\documentclass[twoside]{article}
\pdfoutput=1
 \usepackage[accepted]{aistats2020}

\usepackage[round]{natbib}

\usepackage{nicefrac}
\usepackage[utf8]{inputenc} %
\usepackage[T1]{fontenc}    %
\usepackage[hidelinks]{hyperref}       %
\usepackage{url}            %
\usepackage{booktabs}       %
\usepackage{amsfonts}       %
\usepackage{nicefrac}       %
\usepackage{bbm}
\usepackage{listings}
\usepackage{microtype}      %
\usepackage{algorithmicx,algpseudocode,algorithm}

\usepackage{comment}
\usepackage{enumitem}

\usepackage{appendix}

\usepackage{wrapfig,tikz}
\usepackage{amsmath,longtable,fancyhdr,booktabs,multirow,graphicx,float}
\usepackage{adjustbox, bigstrut, tabularx, multirow, makecell, diagbox}
\usepackage{amssymb,xcolor,amsthm}
\usepackage{subfigure}
\usepackage{color}
\usepackage{colortbl}
\usepackage{theoremref}

\newcommand{\mbb}[1]{\mathbb{#1}}
\newcommand*{\tran}{^{\mkern-1.5mu\mathsf{T}}}

\newcommand{\mcal}{\mathcal}

\newtheorem{lemma}{Lemma}
\newtheorem{theorem}{Theorem}

\newtheorem{definition}{Definition}

\newcommand{\be}{\begin{equation}}
    \newcommand{\ee}{\end{equation}}

\definecolor{Gray}{gray}{0.85}
\definecolor{LightCyan}{rgb}{0.88,1,1}

\newcolumntype{a}{>{\columncolor{Gray}}c}
\newcolumntype{b}{>{\columncolor{white}}c}

\makeatletter
\usepackage{xspace}
\def\@onedot{\ifx\@let@token.\else.\null\fi\xspace}
\DeclareRobustCommand\onedot{\futurelet\@let@token\@onedot}

\newcommand{\figref}[1]{Figure~\ref{#1}}
\newcommand{\algoref}[1]{Algorithm~\ref{#1}}

\newcommand{\secref}[1]{Section~\ref{#1}}
\newcommand{\tabref}[1]{Tab\onedot~\ref{#1}}
\newcommand{\thmref}[1]{Theorem~\ref{#1}}

\newcommand{\bfx}{\mathbf{x}}
\newcommand{\bfA}{\mathbf{A}}

\newcommand{\bfZ}{\mathbf{Z}}
\newcommand{\bfz}{\mathbf{z}}
\newcommand{\bfs}{\mathbf{s}}
\newcommand{\MLP}{\operatorname{MLP}}
\newcommand{\CAT}{\operatorname{CONCAT}}
\newcommand{\bftheta}{{\boldsymbol{\theta}}}
\newcommand{\bfalpha}{{\boldsymbol{\alpha}}}
\newcommand{\bfbeta}{\boldsymbol{\beta}}

\newcommand{\network}{\mathbf{s}}

\def\eg{\emph{e.g}\onedot}

\def\ie{\emph{i.e}\onedot}

\def\aka{a.k.a\onedot}
\def\iid{i.i.d\onedot}
\def\MODEL{EDP-GNN}

\begin{document}

\runningtitle{Permutation Invariant Graph Generation via Score-Based Generative Modeling}

\runningauthor{Chenhao Niu, Yang Song, ... , Stefano Ermon}

\twocolumn[

\aistatstitle{Permutation Invariant Graph Generation via \\
Score-Based Generative Modeling}

\aistatsauthor{ Chenhao Niu$^1$, Yang Song$^2$, Jiaming Song$^2$, Shengjia Zhao$^2$, Aditya Grover$^2$, Stefano Ermon$^2$}

\aistatsaddress{ $^1$Tsinghua University \And $^2$Stanford University} 
]

\begin{abstract}
Learning generative models for graph-structured data is challenging because graphs are discrete, combinatorial, and the underlying data distribution is invariant to the ordering of nodes. However, most of the existing generative models for graphs are not invariant to the chosen ordering, which might lead to an undesirable bias in the learned distribution. To address this difficulty, we propose a permutation invariant approach to modeling graphs, using the recent framework of score-based generative modeling. In particular, we design a permutation equivariant, multi-channel graph neural network to model the gradient of the data distribution at the input graph (\aka, the score function). This permutation equivariant model of gradients implicitly defines a permutation invariant distribution for graphs. We train this graph neural network with score matching and sample from it with annealed Langevin dynamics. In our experiments, we first demonstrate the capacity of this new architecture in learning discrete graph algorithms. For graph generation, we find that our learning approach achieves better or comparable results to existing models on benchmark datasets.

\end{abstract}

\section{INTRODUCTION}

Graphs are used to capture relational structure in many domains, including knowledge bases \citep{hamaguchi2017knowledge}, social networks \citep{hamilton2017inductive,kipf2016semi}, protein interaction networks \citep{fout2017protein}, and physical systems \citep{batagelj2003m}.
Generating graphs using suitable probabilistic models has many applications, such as drug design~\citep{duvenaud2015convolutional,gomez2018automatic,li2018learning}, creating computation graphs for architecture search~\citep{xie2019exploring}, as well as research in network science~\citep{watts1998collective,albert2002statistical,leskovec2010kronecker}.

While many stochastic models of graphs have been proposed, the idea of learning statistical generative models of graphs from data has recently gained significant attention. 
One approach is to use latent variable generative models similar to variational autoencoders~\citep{kingma2013auto}. Examples include GraphVAE~\citep{simonovsky2018graphvae}, Graphite~\citep{grover2018graphite}, and junction tree variational autoencoders~\citep{jin2018junction}. These models typically use a graph neural network (GNN)~\citep{gori2005new,scarselli2008graph} to encode graph data to a latent space, and generate samples by decoding latent variables sampled from a prior distribution. The second paradigm is autoregressive graph generative models~\citep{li2018learning,you2018graph,liao2019lanczosnet}, where graphs are generated sequentially, one node (or one subgraph) at a time. \label{sec:intro}

Although these models have achieved great success, they are not satisfying in terms of capturing the permutation invariance properties of graphs. Permutation invariance is a fundamental inductive bias of graph-structured data. For a graph with $N$ nodes, there are up to $N!$ different adjacency matrices that are equivalent representations of the same graph. Therefore, a graph generative model should ideally assign the same probability to each of these equivalent adjacency matrices. It is challenging, however, to enforce permutation invariance in variational autoencoders or autoregressive models. Some previous approaches only approximately induce permutation invariance: GraphVAE~\citep{simonovsky2018graphvae} uses inexact graph matching techniques requiring up to $O(N^4)$ operations, whereas the model in \cite{li2018learning} augments the training data by randomly permuting the nodes of existing data. Other approaches 
instead focus on selecting a specific node ordering based on heuristics: GraphRNN~\citep{you2018graphrnn} uses random breadth-first search (BFS) to determine an ordering, and GRAN~\citep{liao2019lanczosnet} adaptively chooses an ordering depending on the input graph from a family of pre-defined node orderings.

To better capture the permutation invariance of graphs, we propose a new graph generative model using the framework of score-based generative modeling~\citep{song2019generative}. Intuitively, this approach trains a model to capture the vector field of gradients of the log data density of graphs (\aka, scores). Contrary to likelihood-based models such as variational auto-encoders and autoregressive models, score-based generative modeling imposes fewer constraints on the model architectures (e.g., a score does not have to be normalized). This enables the use of function families with desirable inductive biases, such as permutation invariance. In particular, we leverage graph neural networks~\citep{scarselli2008graph} to build a permutation equivariant model %
for the scores of the distribution over graphs we wish to learn. 
As shown later in the paper, this implicitly defines a permutation invariant distribution over adjacency matrices representing graphs. 

As in other classes of deep generative models, the neural architecture used in score-based generative modeling is critical to its success. In this work, we introduce a new type of graph neural networks, named \MODEL{}, with learnable multi-channel adjacency matrices. 
In our experiments, we first test the effectiveness of \MODEL{} for the task of learning graph algorithms, where it significantly outperforms traditional GNNs. Next, we evaluate the generation quality of our score-based models using MMD \citep{gretton2012kernel} metrics on several graph datasets, where we achieved comparable performance to GraphRNN~\citep{you2018graphrnn}, a competitive method for generative modeling of graphs.

\section{PRELIMINARIES}

\subsection{Notations}
\label{sec:Def}

For each weighted undirected graph, we can choose an ordering of nodes $\pi$ and represent it with an adjacency matrix $\bfA^\pi$. Here we use the superscript $\pi$ to indicate that the rows/columns of $\bfA^\pi$ are arranged in accordance with a specific node ordering $\pi$. When the graph is undirected, the corresponding adjacency matrix $\bfA^\pi$ is symmetric. We denote the set of adjacency matrices as $\mcal{A} = \{\bfA \in \mbb{R}^{N\times N} \mid \bfA=\bfA\tran, N \in \mbb{N}^{+} \}$. 

A distribution of graphs can be represented as a distribution of adjacency matrices $p(\bfA^\pi)$. Since graphs are invariant to permutations, $\bfA^{\pi_1}$ and $\bfA^{\pi_2}$ always represent the same graph for any different node orderings $\pi_1$ and $\pi_2$. This permutation invariance also implies that $\forall \pi_1 \neq \pi_2: p(\bfA^{\pi_1}) = p(\bfA^{\pi_2})$, \ie, the distribution of adjacency matrices is invariant to node permutations. In the sequel, we often omit the superscript $\pi$ in $\bfA^\pi$ when not emphasizing any specific node ordering.

\subsection{Graph Neural Network (GNN)}

Graph neural networks are a family of neural networks that map graphs to vector representations using message-passing type operations on node features~\citep{gori2005new,scarselli2008graph}. They are natural models for graph-structured data; for example, GIN \citep{xu2018powerful} is one type of GNN that is proved to be as expressive as the Weisfeiler-Lehman graph isomorphism test (WL-test). The message passing mechanism guarantees that the output representation $f_\text{GNN}(\bfA^\pi)$ of an input adjacency matrix $\bfA^\pi$ is \emph{equivariant to permutations} of the node ordering $\pi$.%

\subsection{Score-Based Generative Modeling}
Score-based generative modeling~\citep{song2019generative} is a class of  generative models. For a probability density function $p(\bfx)$, the score function is defined as $\nabla_\bfx \log p(\bfx)$. Instead of directly modeling the density function of the data distribution $p_\text{data}(\bfx)$, score-based generative modeling estimates the data score function $\nabla_\bfx \log p_\text{data}(\bfx)$.
The advantage is that \emph{the score function can be easier to model than the density function}.

For better score estimation, following~\citep{song2019generative} we perturb the data with Gaussian noise of different intensities, and estimate the scores jointly for all noise levels. 
We train a noise conditional model $\bfs_\bftheta (\bfx; \sigma)$ (e.g., a neural network parameterized by $\bftheta$) to approximate the score function corresponding to noise level $\sigma$. Given a data distribution $p_\text{data}(\bfx)$, a noise distribution $q_\sigma(\tilde{x} \mid \bfx)$ (\eg, $\mcal{N}(\tilde{\bfx} \mid \bfx, \sigma^2 I)$), and a sequence of noise levels $\{\sigma_i\}_{i=1}^L$, the training loss $\mcal{L}(\bftheta; \{\sigma_i\}_{i=1}^L)$ is defined as:
\begin{align}
 \sum_{i=1}^L \frac{\sigma_i^2}{2L} \mathbb{E}\left[ \| \bfs_\bftheta(\tilde{\bfx}, \sigma_i) - \nabla_{\tilde{\bfx}} \log q_{\sigma_i}(\tilde{\bfx} \mid \bfx)\|_2^2 \right].
\end{align}
where the expectation is taken with respect to the sampling process: $\bfx \sim p_\text{data}(\bfx), \tilde{\bfx} \sim q_{\sigma_i}(\tilde{\bfx} \mid \bfx)$.
We note that all expectations in $\mcal{L}(\bftheta; \{\sigma_i\}_{i=1}^L)$ can be estimated with \iid samples from $p_\text{data}(\bfx)$ and $q_\sigma(\cdot \vert \bfx)$, which are easy to obtain. The objective is $\min\limits_{\bftheta}\mcal{L}(\bftheta; \{\sigma_i\}_{i=1}^L)$.

After the conditional score model $s_\bftheta(\bfx; \sigma)$ has been trained, we use annealed Langevin dynamics~\citep{song2019generative} for sample generation (see \algoref{alg:anneal}).
\begin{algorithm}[t]
	\caption{Annealed Langevin dynamics sampling.}
	\label{alg:anneal}
	\begin{algorithmic}[1]
	    \Require{$\{\sigma_i\}_{i=1}^L, \epsilon, T$} \Comment{$\epsilon$ is smallest step size; $T$ is the number of iteration for each noise level.}
	    \State{Initialize $\tilde{\bfx}_0$}
	    \For{$i \gets 1$ to $L$}
	        \State{$\alpha_i \gets \epsilon \cdot \sigma_i^2/\sigma_L^2$} \Comment{$\alpha_i$ is the step size.}
            \For{$t \gets 1$ to $T$}
                \State{Draw $\bfz_t \sim \mcal{N}(0, I)$}
                \State{$\tilde{\bfx}_{t} \gets \tilde{\bfx}_{t-1} + \dfrac{\alpha_i}{2} \bfs_\bftheta(\tilde{\bfx}_{t-1}, \sigma_i) + \sqrt{\alpha_i}~ \bfz_t$}
            \EndFor
            \State{$\tilde{\bfx}_0 \gets \tilde{\bfx}_T$}
        \EndFor
        \item[]
        \Return{$\tilde{\bfx}_T$}
	\end{algorithmic}
\end{algorithm}

\section{SCORE-BASED GENERATIVE MODELING FOR GRAPHS}

Contrary to the weighted graphs we used to define the probability density function in \secref{sec:Def}, in real-world problems unweighted graphs are much more common, which means entries in the adjacency matrix $\bfA$ can only be either 0 or 1. 
While the score-based method~\citep{song2019generative} was initially proposed for handling continuous data, it can be adopted to generate discrete ones as well. Below, we first show our modifications of score-based generative modeling for graph generation, and then introduce our specialized neural network architecture \MODEL{} for the noise conditional model $\bfs_\bftheta (\bfA; \sigma)$, where $\bfs_\bftheta (\cdot; \sigma): \mcal{A} \rightarrow \mcal{A}$. 

\subsection{Noise Distribution}
We add Gaussian perturbations to adjacency matrices and define the noise distribution $q_{\sigma}(\tilde{\bfA} \mid \bfA)$ as follows
\begin{multline}
  \resizebox{0.9\linewidth}{!}{$\displaystyle
  \begin{cases}
  \prod_{i < j} \frac{1}{\sqrt{2\pi}\sigma} \exp \bigg\{-\frac{(\tilde{\bfA}_{[i,j]} - \bfA_{[i,j]})^2}{2\sigma^2}\bigg\}, &~~ \text{if $\tilde{\bfA} = \tilde{\bfA}\tran$}\\
  0, &~~ \text{otherwise}.
  \end{cases}$}
  \label{eq:noise}
\end{multline}
Intuitively, we only add Gaussian noise to the upper triangular part of the adjacency matrix, because we focus on undirected graphs whose adjacency matrices are symmetric.

Since $\nabla_{\tilde{\bfA}} \log q_{\sigma}(\tilde{\bfA} | \bfA)=-(\tilde{\bfA}-\bfA) / \sigma^{2}$, the training loss of $\network_{\bftheta}(\bfA, \sigma)$ is
\begin{multline} \label{eq:loss}
\resizebox{0.87\columnwidth}{!}{$\displaystyle
\mcal{L}(\bftheta; \{\sigma_i\}_{i=1}^L) \triangleq \frac{1}{2L}\sum_{i=1}^L \sigma_i^2 \mathbb{E}\bigg[
  \bigg\|\network_{\bftheta}(\tilde{\bfA}, \sigma)+\frac{\tilde{\bfA}-\bfA}{\sigma^{2}}\bigg\|_{2}^{2}\bigg]$}.
\end{multline}
where the expectation is over the sampling process defined via $\bfA \sim p_{\text {data }}(\bfA)$ and $\tilde{\bfA} \sim q_{\sigma}(\tilde{\bfA}|\bfA)$. The objective is $\min\limits_{\bftheta}\mcal{L}(\bftheta; \{\sigma_i\}_{i=1}^L)$.

Note that the supports of the noise distributions $\{q_{\sigma_i}\}_{i=1}^L$ span $\mbb{R}^{N\times N}$, where $N$ is the number of nodes of the input graph. Therefore, the scores of perturbed distributions corresponding to all noise levels are well-defined, regardless of whether the training samples are discrete or not.

\subsection{Sampling}

To generate $\tilde{\bfA}$, we first sample $N$, which is the number of nodes to be generated, and then sample $\tilde{\bfA} \in \mbb{R}^{N\times N}$ with annealed Langevin dynamics. This amounts to factorizing $p(\bfA) = \sum_{N=1}^{\infty} p(\bfA \mid \bfA \in \mbb{R}^{N\times N}) p(N)$. Implementation-wise, we sample $N$ from the empirical distribution of number of nodes in the training dataset, as done in \citep{li2018multi}. When doing annealed Langevin dynamics, we first initialize $\tilde{\bfA}_0$ using folded normal distributions, \ie,
\begin{align*}
(\tilde{\bfA}_0)_{[i,j]} = \begin{cases}
|\varepsilon_{[i,j]}|, & i< j \\
(\tilde{\bfA}_0)_{[j,i]}, & \text{otherwise},
\end{cases}
\end{align*}
where all $\varepsilon_{[i,j]} \sim \mathcal{N}(0, 1)$. %
Then, we update $\tilde{\bfA}$ by iteratively sampling from a series of trained conditional score models $\{\bfs_\bftheta(\bfA;\sigma_i)\}_{i=1}^L$ using Langevin dynamics. For each of the conditional score model $\bfs_\bftheta(\bfA;\sigma_i)$, we run Langevin dynamics for $T$ steps, where the series $\{\sigma_i\}_{i=1}^L$ is annealed down over the process such that $\sigma_1$ is large but $\sigma_L$ is small enough that it can be ignored. %
As a minor modification, we change the noise term $\bfz_t$ in \algoref{alg:anneal} to a symmetric one $\tilde{\bfz_t}$, given by 
\begin{align*}
(\tilde{\bfz}_t)_{[i,j]} = \begin{cases}
(\tilde{\bfz}_t)_{[i,j]}, & i< j \\
(\tilde{\bfz}_t)_{[j,i]}, &  i \geq j,
\end{cases}
\end{align*}
which accouts for the symmetry of adjacency matrices.

Score-based generative modeling provides samples in the continuous space, whereas graph data are often discrete. In order to obtain discrete samples, we quantize the generated continuous adjacency matrix (denoted as $\tilde{\bfA}$) to a binary one (denoted as $\bfA^{(\mathrm{sample})}$) at the end of annealed Langevin dynamics. Formally, this quantization operation is defined as
\begin{equation}\label{eq:quantize}
  \bfA^{(\mathrm{sample})}_{[i,j]} = \mathbbm{1}_{\tilde{\bfA}_{[i,j]} >0.5}
\end{equation}
where $\mathbbm{1}$ is an indicator function that evalutes to 1 when the condition holds and 0 otherwise.

\subsection{Permutation Equivariance and Invariance}

Permutation invariance is a desirable property of graph generative models, since the true distribution $p_{\text{data}}(\bfA)$ is inherently permutation invariant. We show that by using a permutation equivariant score function $\bfs_\bftheta (\bfA; \sigma)$, the corresponding distribution is permutation invarant.

\begin{theorem}\label{th:perm-inv}
	If $\network: \mathbb{R}^{N\times N} \rightarrow \mathbb{R}^{N\times N}$ is a permutation equivariant function, then the scalar function $f_{\network}= \int_{\gamma[\mathbf{0}, \bfA]}\langle\network(\mathbf{X}), \operatorname{d}\mathbf{X}\rangle_{\mathrm{F}} + C$ is permutation invariant, where $\langle\bfA, \mathbf{B}\rangle_{\mathrm{F}}=\operatorname{tr}({\bfA^\intercal} \mathbf{B})$ is the Frobenius inner product, $\gamma[\mathbf{0}, \bfA]$ is any curve from $\mathbf{0} = \{0\}_{N\times N}$ to $\bfA$, and $C \in \mathbb{R}$ is a constant.
\end{theorem}
\begin{proof}
	See Appendix~\ref{app:proof}.
\end{proof}
Since the gradient of log-likelihood estimation $\network_\bftheta(\bfA) = \nabla_\bfA \log p_\bftheta(\bfA)$ is permutation equivariant, the implicitly defined log-likelihood function $\log p_\bftheta(\bfA)$ is permutation invariant, according to \thmref{th:perm-inv}, given below by the line integral of $\network_\bftheta(\mathbf{X})$.
\begin{displaymath}
  \log p_\bftheta(\bfA) = \int_{\gamma[\mathbf{0}, \bfA]}\langle\network_\bftheta(\mathbf{X}), \operatorname{d}\mathbf{X}\rangle_{\mathrm{F}} + \log p_\bftheta(\mathbf{0})
\end{displaymath}

\subsection{Edgewise Dense Prediction Graph Neural Network (\MODEL)}

\begin{figure*}[t]
\centering
	\includegraphics[width=0.88\linewidth]{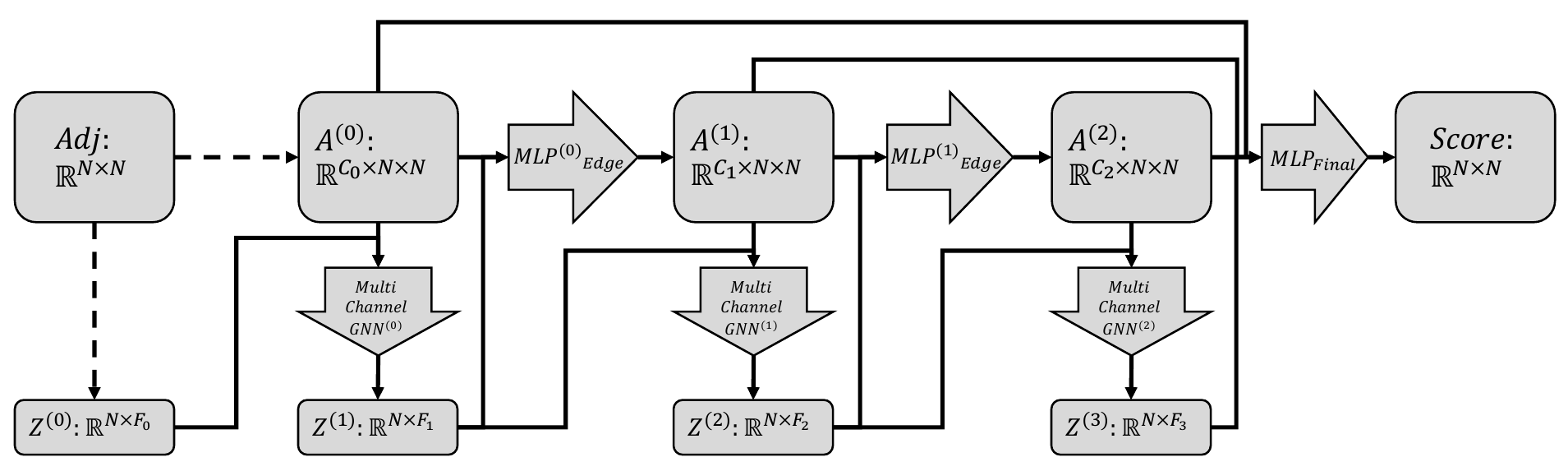}
\caption{This figure shows an \MODEL\space with three layers. The input is an adjacency matrix of a graph with $N$ nodes given a fixed node ordering, and the outputs are edge representations. The dashed lines are preprocessing steps, and solid lines represent network computations.}
\label{fig:network}
\end{figure*}

Below, we introduce a GNN-based score network $\bfs_\bftheta (\bfA; \sigma)$ that can effectively model the scores of graph distributions while being permutation equivariant.

\subsubsection{Multi-Channel GNN Layer}

We introduce the multi-Channel GNN layer , an extended version of the GIN \citep{xu2018powerful} layer, which serves as a basic component of our \MODEL \space model. The intuition is to run message-passing simultaneously on many different graphs, and collect the node features from all the channels via concatenation. For a $C$-channel GNN layer with $M$ message-passing steps, the $m$-th message-passing step can be expressed as follows,
\begin{align*}
	 \tilde{\bfZ}^{(m+1)}_{[c, \cdot]} &= \bfA^{(k)}_{[c, \cdot, \cdot]} \bfZ^{(m)}_{[\cdot]}, \quad \text{for } c = 0, 1, \dots, C-1,\\
	 \bfZ^{(m+1)}_i &= \MLP^{(m)}_{\text{Node}}\bigg(\CAT \bigg(  \\  
	 &\tilde{\bfZ}^{(m+1)}_{[c, i]} + (1+\epsilon) \bfZ^{(m)}_i | c=0, \dots,C-1  \bigg)\bigg), 
\end{align*}
where $i$ is the index of nodes, $C$ is the number of channels, $\bfA^{(k)} \in \mathbb{R}^{C \times N \times N}$ is the multi-channel adjacency matrix, and $\bfZ^{(m)} \in \mathbb{R}^{N \times F^{(m)}}$ is the vector of node features. Here $\epsilon$ is a learnable parameter, the same as in the original GIN, $\CAT$ stands for the concatenation operation, and $\MLP^{(m)}_{\text{Node}}$ transforms each node feature using a multilayer perceptron.

After $M$ steps of message-passing, we use the same concatenation operation as GIN to obtain node features. Specifically, for each node $v_i$, the output feature is given by
\begin{displaymath}
  (\bfZ_{\text{out}})_i = \CAT(\bfZ^{(m)}_i | m=0,1,\dots, M-1).
\end{displaymath}

Henceforth, we denote our Multi-Channel GNN layer as
 $$\bfZ_{\text{out}} = \operatorname{MultiChannelGNN}(\bfA, \bfZ_{\text{in}}). $$

\subsubsection{\MODEL\space Layer}

The \MODEL\space layer is the key component of our model. It transforms the input adjacency matrix to another one, allowing us to adaptively change the process of message passing. The intuition is similar to neural networks for image dense prediction tasks (\eg, semantic parsing), where convolutional layers transform the input image to a feature map in a pixelwise manner, leveraging local information around each pixel location. Similarly, we want our GNN layer to extract edgewise features and map them to a new adjacency matrix, using local information (which is defined in terms of connectivity) of each node in the graph.

One \MODEL\space layer has two steps:

\begin{enumerate}
\item \textbf{Node feature inference:} Using MultiChannelGNN to encode the local structure of different channels of the graph into node features, given by
\begin{equation}
  \bfZ^{(k+1)} = \operatorname{MultiChannelGNN}^{(k)}(\bfA^{(k)}, \bfZ^{(k)});
\end{equation}

\item \textbf{Edge feature inference:} Updating the feature vector of each edge based on the current features of the edge and the updated feature vector of the two endpoints. For each edge $\{v_i, v_j\}$, this operation is given by
\begin{multline*}
\resizebox{0.9\columnwidth}{!}{$
\mathbf{\tilde{A}}^{(k+1)}_{[\cdot, i,j]} = \MLP^{(k)}_{\text{Edge}}\bigg(
\CAT(\bfA^{(k)}_{[\cdot,i,j]}, \bfZ^{(k+1)}_i, \bfZ^{(k+1)}_j)\bigg)$},	
\end{multline*}
where $\MLP^{(k)}_{\text{Edge}}$ denotes a multilayer perceptron applied to edge features. To ensure symmetry, the new adjacency matrix is given by
\begin{equation}
\bfA^{(k+1)} = \mathbf{\tilde{A}}^{(k+1)} + (\mathbf{\tilde{A}}^{(k+1)})\tran.
\end{equation}
\end{enumerate}

\subsubsection{Input and Output Layers}

\textbf{Input layer:} Input graphs need to be preprocessed before they can be fed into our \MODEL\space model. In particular, we take adjacency matrices of two channels as the input, where the first channel is the original adjacency matrix of an input graph, and the other channel is the negated version of the same adjacency matrix, where each entry is flipped. The node features are initialized using the weighted degrees. Formally,
\begin{displaymath}
  \begin{aligned}
	\bfZ^{(0)}_i &= \sum_j\mathbf{Adj}_{[i, j]}, \forall v_i \in \mathcal{V} \\
	\bfA^{(0)}_{[0, \cdot, \cdot]} &= \mathbf{Adj} \\
	\bfA^{(0)}_{[1, \cdot, \cdot]} &= 1 - \mathbf{Adj}
	\end{aligned}
\end{displaymath}
where $\mathbf{Adj}$ is the adjacency matrix of an input graph. If we have node features $\mathbf{X} \in \mathbb{R}^{N\times F_0}$ from data, then we use the following initialization for each node $v_i$
\begin{displaymath}
  \bfZ^{(0)}_i = \CAT\bigg(\mathbf{X}_i, \sum_j \mathbf{Adj}_{i, j}\bigg).
\end{displaymath}
\textbf{Output layer:} To get the output, we employ a similar approach to \cite{xu2018representation}, where we aggregate the information from all previous layers to produce a set of permutation equivariant edge features. This can effectively collect information extracted in shallower layers.
Formally, for each edge $\{v_i, v_j\}$, the output features are given by
\begin{align*}
\resizebox{\columnwidth}{!}{$\displaystyle
\network_\bftheta(\bfA)_{[i,j]} = \MLP_{\text{final}}\left(\CAT\left(\bfA^{(k)}_{[\cdot, i, j]}| k=0, \dots, K-1 \right)\right)$}.
\end{align*}

\subsubsection{Noise Level Conditioning}

The framework of score-based generative modeling proposed in \citep{song2019generative} requires a score network conditioned on a series of noise levels. We hope to provide the conditioning on noise levels with as few extra parameters as possible. To this end, we add gains and bias terms conditioned on the index $i$ of the noise level $\sigma_i$ in all MLP layers, and share all the parameters across different noise levels. %
A conditional MLP layer for $\network_{\bftheta}(\bfA, \sigma_i)$ is denoted as
\begin{displaymath}
 	f_i(\bfA) = \operatorname{activate}((\mathbf{W}\bfA + \mathbf{b})\bfalpha_i + \bfbeta_i)
\end{displaymath}
where $\bfalpha_i, \bfbeta_i$ are learnable parameters for each noise level $\sigma_i$ and $\operatorname{activate}(\cdot)$ denotes the activation function. We empirically found that this implementation of noise conditioning achieves similar performance to  separately training a score network for each noise level. 

\subsubsection{Permutation Equivariance of \MODEL{}}

The message passing operations in a graph neural network are guaranteed to be permutation equivariant \citep{keriven2019universal}, as well as edgewise and nodewise operations for graphs. Since operations in \MODEL\space are either message passing or edgewise/nodewise transformations, the edge features produced by \MODEL{} are guaranteed to be permutation equivariant. In the last \MODEL{} layer, each edge feature is one component of the estimated score. Hence \thmref{th:perm-inv} applies to this score network.

\section{RELATED WORK}

\paragraph{Flow-Based Graph Generative Models} In addition to models mentioned in \secref{sec:intro}, there is also an emerging class of graph generative models based on invertible mappings, such as GNF~\citep{liu2019graph} and GraphNVP~\citep{madhawa2019graphnvp}. These models modify the architecture of a graph neural network (GNN) using coupling layers~\citep{dinh2016density} to enable maximum likelihood learning via the change of variables formula. Since GNNs are permutation invariant, both GNF and GraphNVP could be permutation invariant in principle. However, GraphNVP opts not to be permutation invariant because making their model fully permutation invariant hurts the empirical performance. In contrast, GNF is a permutation invariant model. It achieves permutation invariance by first using a permutation equivariant auto-encoder to encode the graph structure into a set of node features, and then model the distribution of the node features using reversible graph neural networks.

\paragraph{GNNs that Learn Edge Features} Although the majority of GNNs focus on node feature learning, (\eg, node classification tasks), there are GNNs, prior to our \MODEL{}, that have intermediate edge features as well. For example, Graph Attention Networks \citep{velivckovic2017graph} compute an attention coefficient for each edge during message passing (MP) steps. \cite{gong2019exploiting} further explored methods to utilize edge features during the MP steps, such as using normalized attention coefficients to construct a new adjacency matrix for the next MP step, and passing the message simultaneously on multi-input adjacency matrices. However, the model in \cite{gong2019exploiting} is not designed for predicting edge features, and the capability to make edgewise prediction is limited by the normalizing operation and the restrictive form of attentions. \cite{kipf2018neural} proposed a GNN-based VAE model for relational inference for interacting systems. Contrary to their model which predicts edge information based on only node features, our model takes a weighted graph without node features.

\raggedbottom

\section{EXPERIMENTS}

\begin{figure*}[t]
\centering
\includegraphics[width=0.88\linewidth]{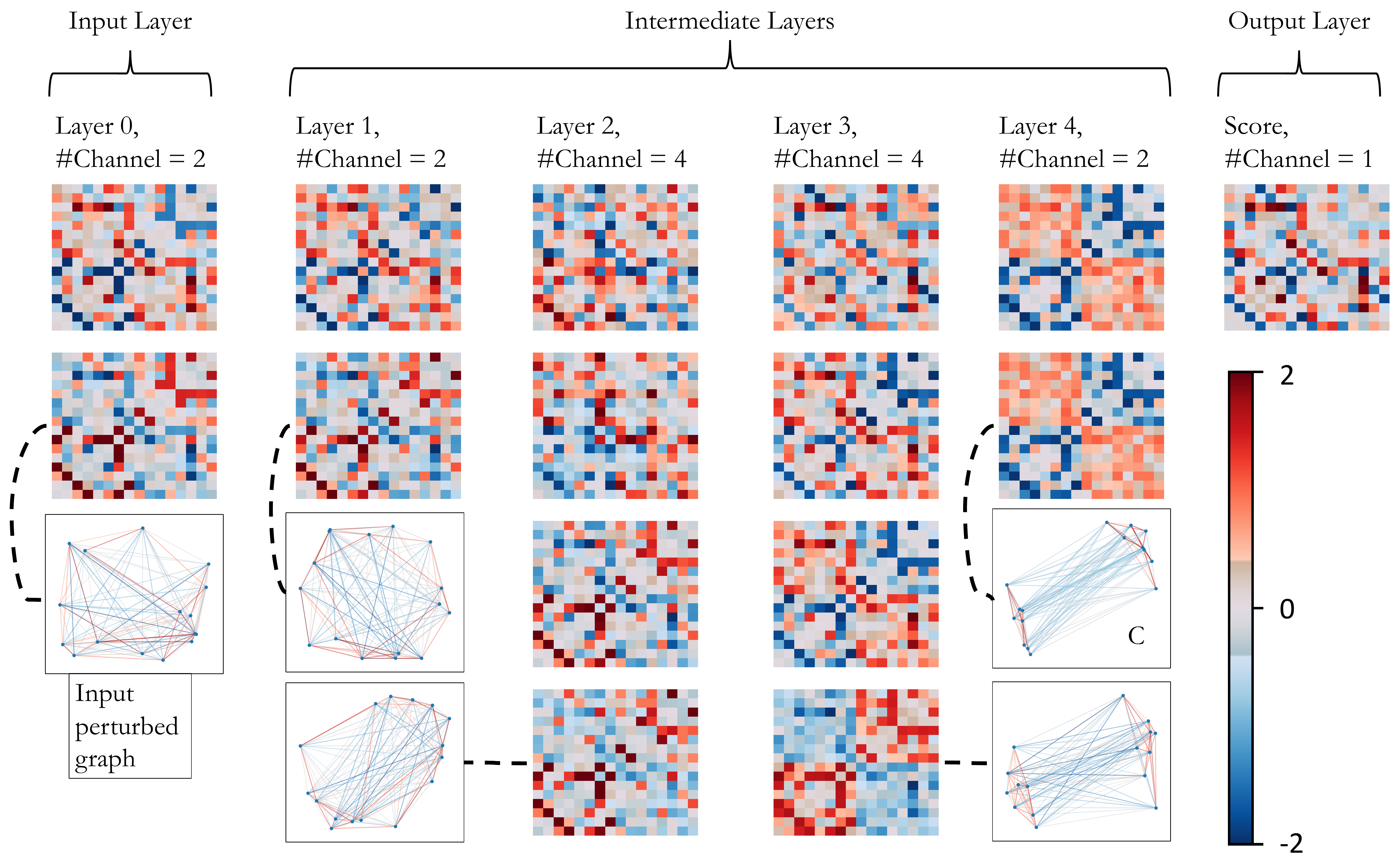}
\caption{Visualization of channels for a pre-trained \MODEL\space model on the Community-small dataset. The model is trained with a single noise level $\sigma=0.6$. The input is a community graph, but perturbed with Gaussian noise with $\sigma = 0.6$. The edge weights of each adjacency matrix are standardized to zero mean and unit variance. Since our model is agnostic to different permutations of nodes, we chose a specific ordering so that the adjacency matrices of community graphs possess a block diagonal form.
We visualize one adjacency matrix for each layer. Sometimes a graph is less visually interpretable, and we instead visualize its complementary graph and mark it with "C". By comparing the graph visualizations for the 3rd, 4th, and the input layers, we observe that the model maps the perturbed graph with no visible structures to a graph with clear "community" structures.}
\label{fig:channel}
\end{figure*}

\subsection{Learning Graph Algorithms}

In this section, we empirically demonstrate the power of the proposed \MODEL\space model on edgewise prediction tasks. %
In particular, we reduce several classic graph algorithms to the task of predicting whether each edge is in the solution set or not. The training data include a graph and the corresponding solution set, and we train our models to fit the solution set by minimizing the cross-entropy loss. %

\paragraph{Setup} To verify the ability of \MODEL\space of making edgewise dense predictions, we tested \MODEL\space on learning classic graph algorithms, by labeling all the edges in a graph to indicate whether an edge is in the solution set or not. We choose two simple tasks, 1) Shortest Path (SP) between a given pair of nodes, and 2) Maximum Spanning Tree (MST) of a given graph. The solution set of SP corresponds to a path connecting the pair of nodes with the shortest length, while the solution set of MST is the collection of all edges inside the maximum spanning tree. For both tasks, all the graphs are randomly sampled from the Erd\H{o}s and R{\'e}nyi model (E-R) \citep{erdos1960evolution} with $n = 12$ and $p = 0.3$. For weighted graphs, all the edge weights are uniformly sampled from $[0,1]$. A prediction is considered correct if and only if all the labels of the graph are correct. We calculate the accuracy over a fixed test set as the metric. For the baseline model, we use vanilla GIN \citep{xu2018powerful}.

\paragraph{Training} During training, we generate the training data dynamically on the fly and use the cross-entropy loss as the training objective for both tasks.%

\begin{table}[t]
\centering
\begin{tabular}{@{}llll@{}}
\toprule
Model  & SP (UW) & SP (W) & MST (W)  \\ \midrule
GIN    & 0.57  & 0.12 & 0.20 \\
\MODEL & \textbf{0.60}  & \textbf{0.92} & \textbf{0.84} \\ \bottomrule
\end{tabular}
\caption{The test set accuracy of \MODEL\space vs. GIN on learning the shortest path (SP) and maximum spanning tree (MST) algorithms. "UW" and "W" stand for "unweighted" and "weighted" respectively. Since the training set is dynamically generated, the performance on (newly generated) training set and test set has no difference. Note that for unweighted graphs, there can be more than one shortest path for a given pair of nodes, and the accuracy is underestimated as we randomly picked one as the ground truth, in which case an accuracy of 0.6 is pretty non-trivial.}
\label{tab:subgraph-res}
\end{table}

\paragraph{Results}
All results are provided in \tabref{tab:subgraph-res}. We observe that \MODEL\space performs similarly to GIN for unweighted graphs, but achieves much better performance when graphs are weighted. This confirms that \MODEL{} is more effective for edgewise predictions.

\subsection{Graph Generation Task}

\begin{figure*}[t]
    \centering
	\subfigure[Training data]{
        \includegraphics[width=.23\linewidth]{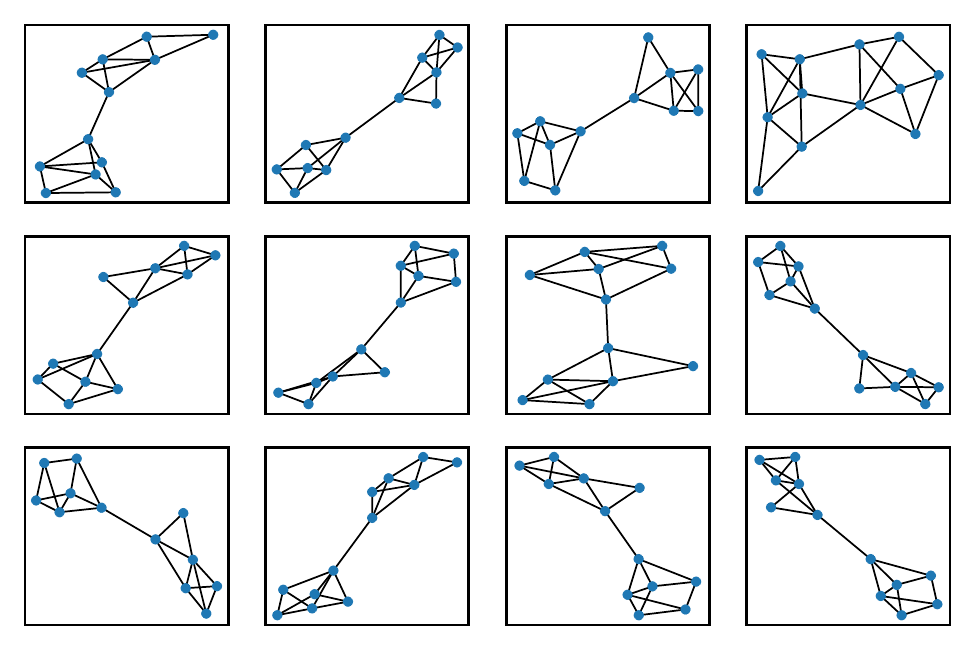}
    } \quad
    \subfigure[\MODEL\space samples]{
        \includegraphics[width=.23\linewidth]{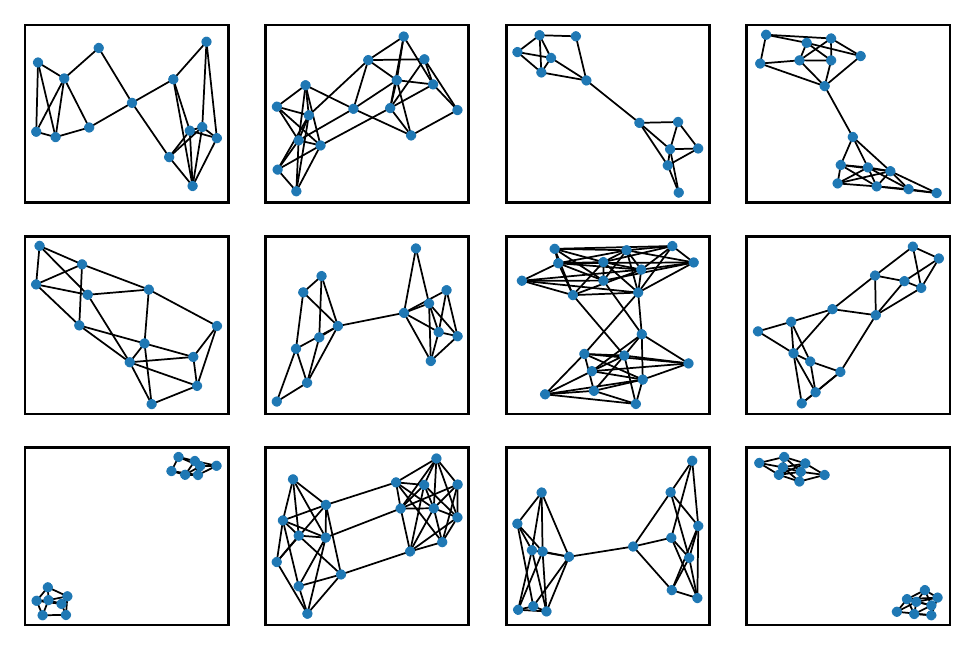}
    }  \quad
    \subfigure[GraphRNN samples]{
        \includegraphics[width=.23\linewidth]{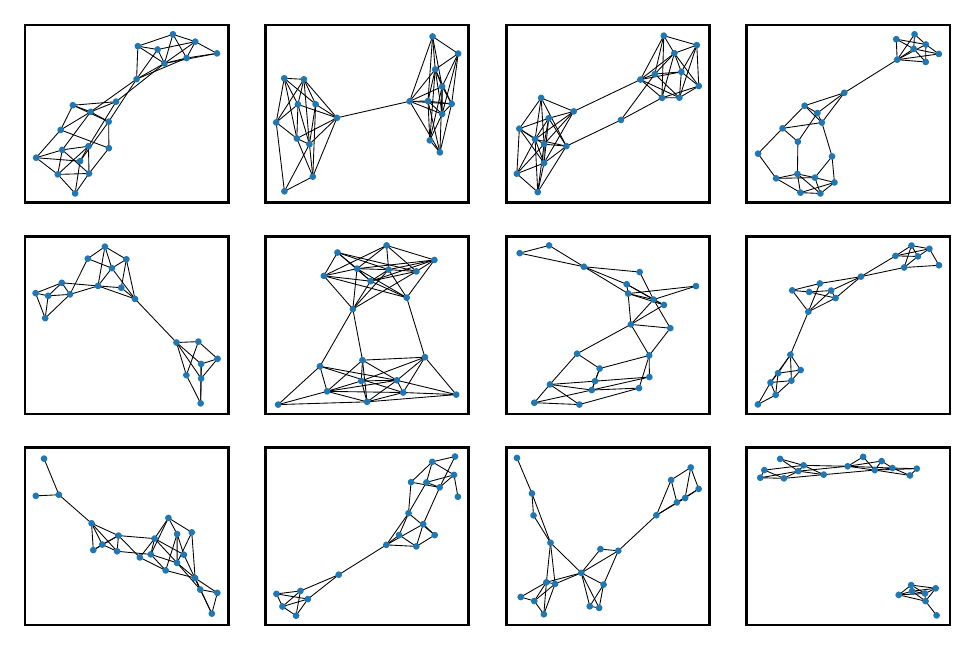}
    }
	\subfigure[Training data]{
        \includegraphics[width=.23\linewidth]{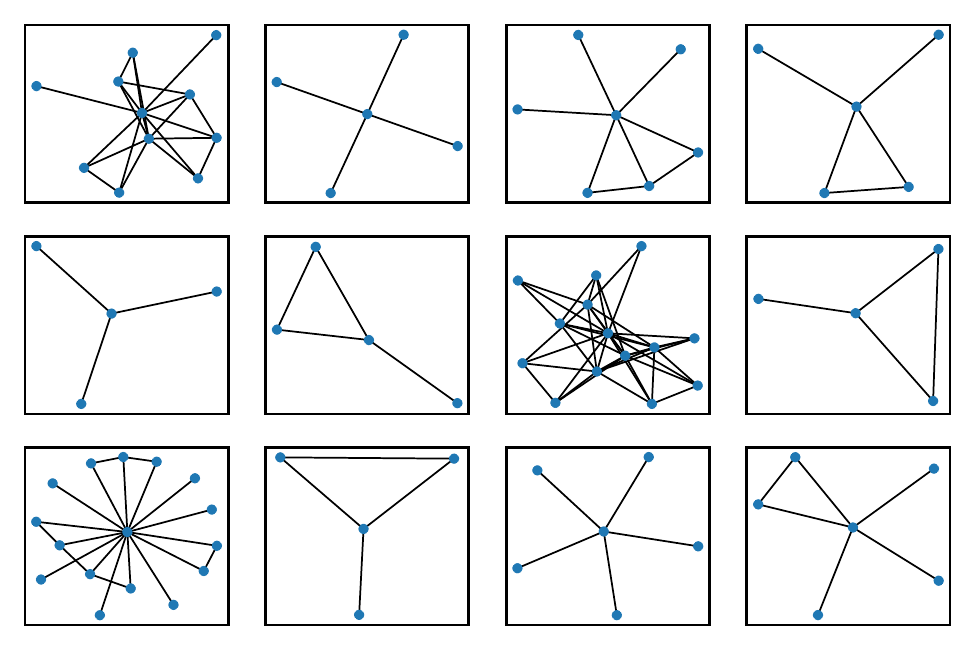}
    } \quad
    \subfigure[\MODEL \space samples]{
        \includegraphics[width=.23\linewidth]{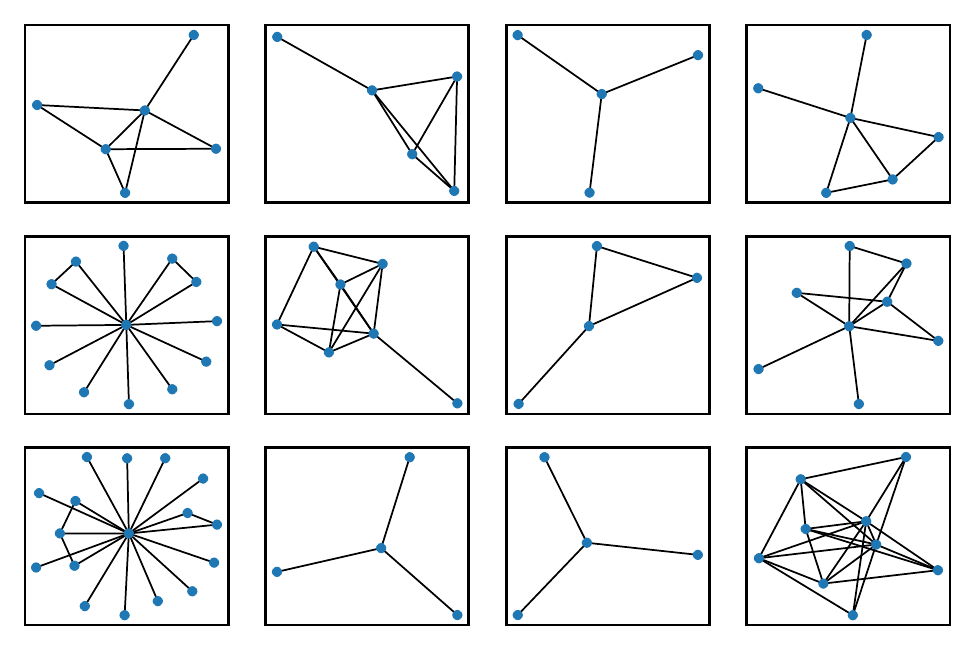}
    }  \quad
    \subfigure[GraphRNN samples]{
        \includegraphics[width=.23\linewidth]{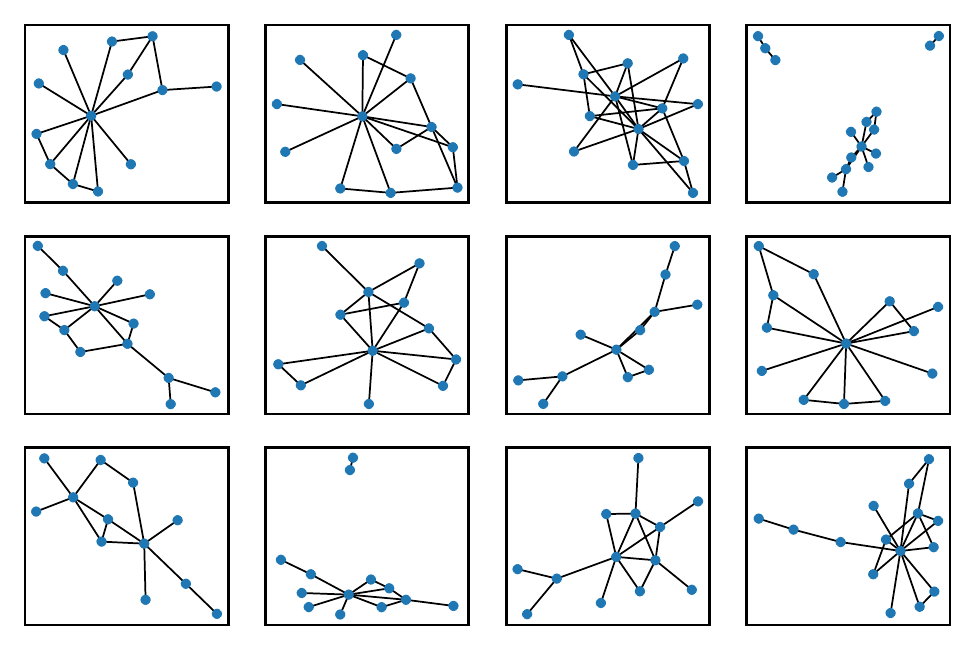}
    }
\caption{Samples from the training data, \MODEL, and GraphRNN, on Community-small (top row) and Ego-small (bottom row).}
\label{fig:mmd-Ego-small}
\end{figure*}

In this section, we demonstrate that our \MODEL{} is capable of producing high-quality graph samples via score-based generative modeling. To better understand learnable multi-channel adjacency matrices in our model, we visualize the intermediate channels in Figure~\ref{fig:channel}, and perform extensive ablation studies.%

\paragraph{Datasets and Baselines} We tested our model on two datasets, Community-small ($12 \leq N \leq 20$) and Ego-small ($4 \leq N \leq 18$), which are also used by \cite{you2018graphrnn}, and \cite{liu2019graph}. See Appendix~\ref{app:exp} for more details. Our baselines include GraphRNN \citep{you2018graphrnn}, Graph Normalizing Flow(GNF) \citep{liu2019graph}, GraphVAE \citep{simonovsky2018graphvae}, and DeepGMG \citep{li2018learning}.

\paragraph{Metrics} To evaluate generation quality, we used maximum mean discrepancy (MMD) over some graph statistics, as proposed by \cite{you2018graphrnn}. We calculated MMD for three graph statistics: 1) degree distribution, 2) cluster coefficient distribution, and 3) the number of orbits with 4 nodes.

\paragraph{Results}
We compare \MODEL{} against baselines and summarize results in \tabref{tab:MMD}. Our model performs comparably to GraphRNN and GNF with respect to most MMD metrics, and outperforms all other methods when considering the overall average of MMDs on two datasets.

\begin{table*}[t]
\centering
\begin{adjustbox}{max width=0.85\textwidth}
\begin{tabular}{@{}llllllllll@{}}
\toprule
\multirow{2}{*}{Model} & \multicolumn{4}{l}{Community-small}                               & \multicolumn{4}{l}{Ego-small}                                     & \multirow{2}{*}{Avg.} \\ \cmidrule(lr){2-9}
                       & Deg.           & Clus.          & Orbit          & Avg.           & Deg.           & Clus.          & Orbit          & Avg.           &                       \\ \midrule
GraphVAE               & 0.350          & 0.980          & 0.540          & 0.623          & 0.130          & 0.170          & 0.050          & 0.117          & 0.370                 \\
DeepGMG                & 0.220          & 0.950          & 0.400          & 0.523          & 0.040          & 0.100          & 0.020          & 0.053          & 0.288                 \\
GraphRNN               & 0.080          & \textbf{0.120} & 0.040          & 0.080          & 0.090          & 0.220          & 0.003          & 0.104          & 0.092                 \\
GNF                    & 0.200          & 0.200          & 0.110          & 0.170          & \textbf{0.030} & 0.100          & \textbf{0.001} & \textbf{0.044} & 0.107                 \\
\MODEL                 & \textbf{0.053} & 0.144          & \textbf{0.026} & \textbf{0.074} & 0.052          & \textbf{0.093} & 0.007          & 0.050          & \textbf{0.062}        \\ \midrule
GraphRNN (1024)         & 0.030          & \textbf{0.010} & \textbf{0.010} & \textbf{0.017} & 0.040          & 0.050          & 0.060          & 0.050          & 0.033                 \\
GNF (1024)              & 0.120          & 0.150          & 0.020          & 0.097          & \textbf{0.010} & 0.030          & \textbf{0.001} & 0.014          & 0.055                 \\
\MODEL{} (1024)          & \textbf{0.006} & 0.127          & 0.018          & 0.050          & \textbf{0.010} & \textbf{0.025} & 0.003          & \textbf{0.013} & \textbf{0.031}        \\ \bottomrule
\end{tabular}
\end{adjustbox}

\caption{MMD results of various graph generative models. Rows marked with (1024) mean the corresponding number of samples is 1024; otherwise, the number of samples equals the size of the test set. Apart from three MMD statistics, we also provide their average values, noted as "Avg.". The rightmost column is the overall average of all MMDs on two datasets. For baselines, we directly ported the results from \cite{you2018graphrnn} and \cite{liu2019graph}. For a fair comparison, we followed the settings of evaluation in \cite{liu2019graph}.}
\label{tab:MMD}
\end{table*}

\subsubsection{Understanding Intermediate Channels}
Intuitively, the intermediate channels of \MODEL\space should be analogous to those in convolutional neural networks (CNN) feature maps. Since channels of feature maps can be visualized as images in CNNs, we propose to visualize each channel of multi-channel adjacency matrices as a graph. The \MODEL{} layers should be able to map an input graph to intermediate graphs that possess interpretable semantics.

In \figref{fig:channel}, we visualize the channels of intermediate adjacency matrices for a \MODEL{} model trained on the Community-small dataset. We observe that the model processes a perturbed community graph with no clearly visible structures to a graph with a structure of two equal-sized communities.

As implied by the training objective \eqref{eq:loss}, the score network $\bfs_\bftheta(\tilde{\bfA}, \sigma)$ can perfectly predict the ground truth score, \ie, $\frac{\tilde{\bfA}-\bfA}{\sigma^2}$, if it can map the noise-perturbed graph $\tilde{\bfA}$ to the true (noise-free) graph $\bfA$ in some of the intermediate channels. Therefore, an ideal score network should be able to 1) understand the structure of a given graph, before 2) mapping a perturbed graph to the corresponding denoised graph. While previous GNNs are designed for the former task, \MODEL\space is especially capable of solving the latter one.

\subsubsection{Ablation Studies}
\begin{table}[ht!]
\centering
\begin{adjustbox}{max width=0.35\textwidth}
\begin{tabular}{@{}llllll@{}}
\toprule
\multirow{2}{*}{A*} & \multirow{2}{*}{C*} & \multicolumn{2}{l}{\textbf{Community-small}}                                                                                                             & \multicolumn{2}{l}{\textbf{Ego-small}}                                                                                                                  \\ \cmidrule(l){3-6} 
                    &                     & \begin{tabular}[c]{@{}l@{}}Train \\ loss\end{tabular} & \begin{tabular}[c]{@{}l@{}}Test \\ loss\end{tabular} & \begin{tabular}[c]{@{}l@{}}Train \\ loss\end{tabular} & \begin{tabular}[c]{@{}l@{}}Test \\ loss\end{tabular}  \\ \midrule
N                   & N                   & 140                                               & 140                                               & 14                                               & 17                                               \\
Y                   & N                   & 120                                               & 120                                              & 12                                               & 15                                               \\
N                   & Y                   & 110                                               & 120                                              & 13                                               & 15                                               \\
Y                   & Y                   & \textbf{98}                                               & \textbf{96}                                              & \textbf{10}                                               & \textbf{12 }                                            \\ \bottomrule
\end{tabular}
\end{adjustbox}
\caption{Ablation experiments on Community-small and Ego-small datasets. The training and test losses are defined by \eqref{eq:loss}. A* indicates whether the adjacency matrix is learnable, and C* indicates whether the intermediate adjacency matrices have multi-channels. }
\label{tab:ablation}
\end{table}

To verify the importance of intermediate adjacency matrices in \MODEL\space to be 1) learnable and 2) multi-channel, we conducted ablative studies on Community-small and Ego-small datasets. We switched on/off the two properties respectively, and provide the performance comparison in \tabref{tab:ablation}. Note that \MODEL{} is equivalent to vanilla GIN when intermediate adjacency matrices are single-channel and non-learnable. As shown in \tabref{tab:ablation}, both properties can improve the expressivity for score modeling, in the sense of reducing the training and test score matching losses. As expected, the performance is optimal when both properties are combined.

\section{CONCLUSION}

We propose a permutation invariant generative model for graphs based on the framework of score-based generative modeling. In particular, we implicitly define a permutation invariant distribution over graph adjacency matrices by modeling the corresponding permutation equivariant score function and sampling with Langevin dynamics. For effective score modeling of graph distributions, we propose a new permutation equivariant GNN architecture, named \MODEL, leveraging trainable, multi-channel adjacency matrices as intermediate layers. Empirically, we demonstrate that \MODEL s \space are more expressive than vanilla GNNs on predicting edgewise features, as evidenced by better performance on the task of learning classic graph algorithms such as shortest paths. Moreover, we show our model can produce samples with quality comparable to existing state-of-the-art models. As one future direction, we hope to improve the scalability of our model by reducing the computational complexity, using techniques such as graph pooling~\citep{ying2018hierarchical}.

\subsubsection*{Acknowledgements}
This research was supported by Intel Corporation, Amazon AWS, TRI, NSF (\#1651565, \#1522054, \#1733686), ONR (N00014-19-1-2145), AFOSR (FA9550-19-1-0024).

\bibliography{main}

\allowdisplaybreaks

\appendix
\onecolumn

\section{EXPERIMENTAL DETAILS}\label{app:exp}

We implement our model using PyTorch \citep{paszke2019pytorch}. The optimization algorithm is Adam \citep{kingma2014adam}. Our code is available at \url{https://github.com/ermongroup/GraphScoreMatching}.

\subsection{Hyperparameters}
For the noise levels $\{\sigma_i\}_{i=1}^L$, we chose $L=6$ and $\{\sigma_i\}_{i=1}^L=[1.6, 0.8, 0.6, 0.4, 0.2, 0.1]$. Empirically, we found those settings work well for all the generation experiments. Note that since all the edge weights in training data (\ie, $\bfA_{i,j}$ in \eqref{eq:noise}) are either 0 or 1, $\sigma_L=0.1$ is small enough for the quantizing operation \eqref{eq:quantize} to prefectly recover the perturbed graph with high probability. 

In the sampling process, we set the number of sampling steps for each noise level to be $T=1000$. Apart from the coefficient $\epsilon$ in step size $\alpha_i = \epsilon \cdot \sigma_i^2/\sigma_L^2$ in Langevin dynamics, we added another scaling coefficient $\epsilon_s$, since it is a common practice of applying Langevin dynamics. We chose the value of the hyper-parameters based on the MMD metrics on the validation set, which contains 32 samples from the training set.
\begin{displaymath}
  \tilde{\bfx}_{t} \gets \tilde{\bfx}_{t-1} + \dfrac{\alpha_i}{2} \bfs_\bftheta(\tilde{\bfx}_{t-1}, \sigma_i) + \epsilon_s\sqrt{\alpha_i}~ \bfz_t
\end{displaymath}

For the network architecture, we used 4 message-passing steps for each GIN, and stacked 5 \MODEL \space layers. The maximum number of channels of all \MODEL \space layer is 4. The maximum size of node features is 16.

\subsection{Dataset}\label{sec:dataset}

\begin{itemize}
  \item \textbf{Community-small:} The graphs are constructed by two equal-sized communities, each of which is generated by E-R model \citep{erdos1960evolution}, with $p=0.7$. For each graph with $N$ nodes, we randomly add $0.05N$ edges between the two communities. The range of total number of nodes per graph is $12 \leq N \leq 20$.
  \item \textbf{Ego-small:} One-hop ego graphs extracted from the Citeseer network \citep{sen2008collective}. The range of node numbers per graph is $4 \leq N \leq 18$.
\end{itemize}

\section{PROPERTIES OF PERMUTATION INVARIANT FUNCTIONS}\label{app:proof}

\subsection{Permutation}

\begin{definition} 
	(Permutation Operation on Matrix) Let $[N] \stackrel{\text { def. }}{=}\{1, \ldots, N\}$. Denote the set of permutations $\pi: [N] \rightarrow [N]$ as $\Pi_N$. The node permutation operation on a matrix $\bfA \in \mathbb{R}^{N \times N}$ is defined by $\bfA^{[\pi]}_{i,j} = \bfA_{\pi(i), \pi(j)}$. 
\end{definition}

\subsection{Permutation Invariant}

\begin{definition}
(Permutation Invariant Function) A function $f$ with $\mathbb{R}^{N\times N}$ as its domain is permutation invariant i.f.f. $\forall \bfA \in \mathbb{R}^{N\times N}, \forall \pi \in \Pi_N, \quad f(\bfA^{[\pi]}) = f(\bfA)$.
\end{definition}

\subsection{Permutation Equivariant}

\begin{definition}
	(Permutation Equivariant Function) A function $\network: \mathbb{R}^{N\times N} \rightarrow \mathbb{R}^{N\times N}$ is permutation equivariant i.i.f. $\forall \bfA\in \mathbb{R}^{N\times N}, \forall \pi \in \Pi_N, \quad \network(A^{[\pi]}) = \left(\network(A)\right)^{[\pi]}$.
\end{definition}

\subsection{Relationship between Permutation Invariance and Permutation Equivariance}

\begin{definition}
(Implicitly Defined Scalar Function)
A function $\network: \mathbb{R}^{N\times N} \rightarrow \mathbb{R}^{N\times N}$ defines a gradient vector field on $\mathbb{R}^{N\times N}$. Veiw $\network$ as the gradient of a scalar value function $f_{\network}(\bfA): \mathbb{R}^{N\times N} \rightarrow \mathbb{R}$. Define $f_{\network}(\bfA) = \int_{\gamma[\mathbf{0}, \bfA]}\langle\network(\mathbf{X}), \operatorname{d}\mathbf{X}\rangle_{\mathrm{F}} + C$, where $\mathbf{0} = \{0\}_{N\times N}$, $\gamma[\mathbf{0}, \bfA]$ is any curve from $\mathbf{0}$ to $\bfA$ and $C \in \mathbb{R}$ is a constant. 
\end{definition}

Under this definition, a vector function $\network$ defines a scalar function $f_{\network}$ implicitly. 

\begin{lemma}
	(Permutation Invariance of Frobenius Inner Product) For any $A, B \in \mathbb{R}^{N\times N}$, the Frobenius inner product of $A,B$ is $\langle\bfA, \mathbf{B}\rangle_{\mathrm{F}}=\sum_{i, j} {A_{i j}} B_{i j}=\operatorname{tr}({\bfA^{T}} \mathbf{B})$. Frobenius inner product operation is permutation invariant, \ie, $\forall \pi \in \Pi_N, \quad \langle\bfA^{[\pi]}, \mathbf{B}^{[\pi]}\rangle_{\mathrm{F}} = \langle\bfA, \mathbf{B}\rangle_{\mathrm{F}}$. %
\end{lemma}

\subsection{Proof of \thmref{th:perm-inv}}\label{sec:proof}

\begin{proof}
	\begin{align*}
		& \forall \bfA \in \mathbb{R}^{N\times N}, \forall \pi \in \Pi_N, \\
		& f(\bfA^{[\pi]}) - f(\mathbf{0}^{[\pi]}) \\
		=& \int_{\gamma[\mathbf{0}^{[\pi]}, \bfA^{[\pi]}]}\langle\network(\mathbf{X}), \operatorname{d}\mathbf{X}\rangle_{\mathrm{F}} \\
		=& \int_{\gamma[\mathbf{0}, \bfA]}\langle\network(\mathbf{X}^{[\pi]}), \operatorname{d}\left(\mathbf{X}^{[\pi]}\right)\rangle_{\mathrm{F}} \\
		=& \int_{\gamma[\mathbf{0}, \bfA]}\langle\left(\network(\mathbf{X})\right)^{[\pi]}, \left(\operatorname{d}\mathbf{X}\right)^{[\pi]}\rangle_{\mathrm{F}} \\		
		=& \int_{\gamma[\mathbf{0}, \bfA]}\langle\network(\mathbf{X}), \operatorname{d}\mathbf{X}\rangle_{\mathrm{F}} \\
		=& f(\bfA) - f(\mathbf{0}) \\ 
		& \ie \quad f(\bfA^{[\pi]}) = f(\bfA)
	\end{align*}
\end{proof}

\pagebreak

\section{EXTRA SAMPLES}

\begin{figure*}[ht]
    \centering
	\subfigure[Training data]{
        \includegraphics[width=.27\linewidth]{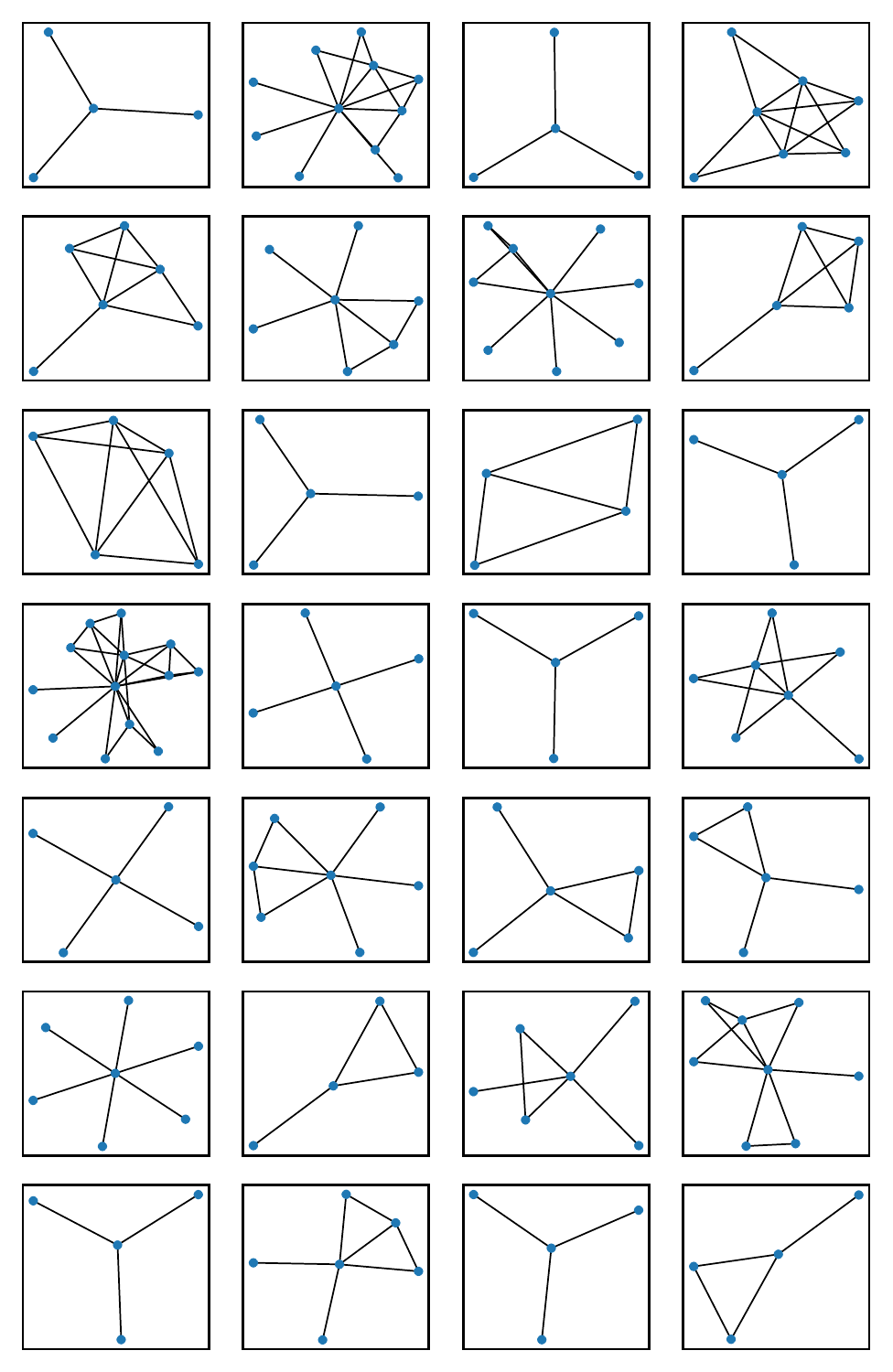}
    } \quad
    \subfigure[\MODEL \space samples]{
        \includegraphics[width=.27\linewidth]{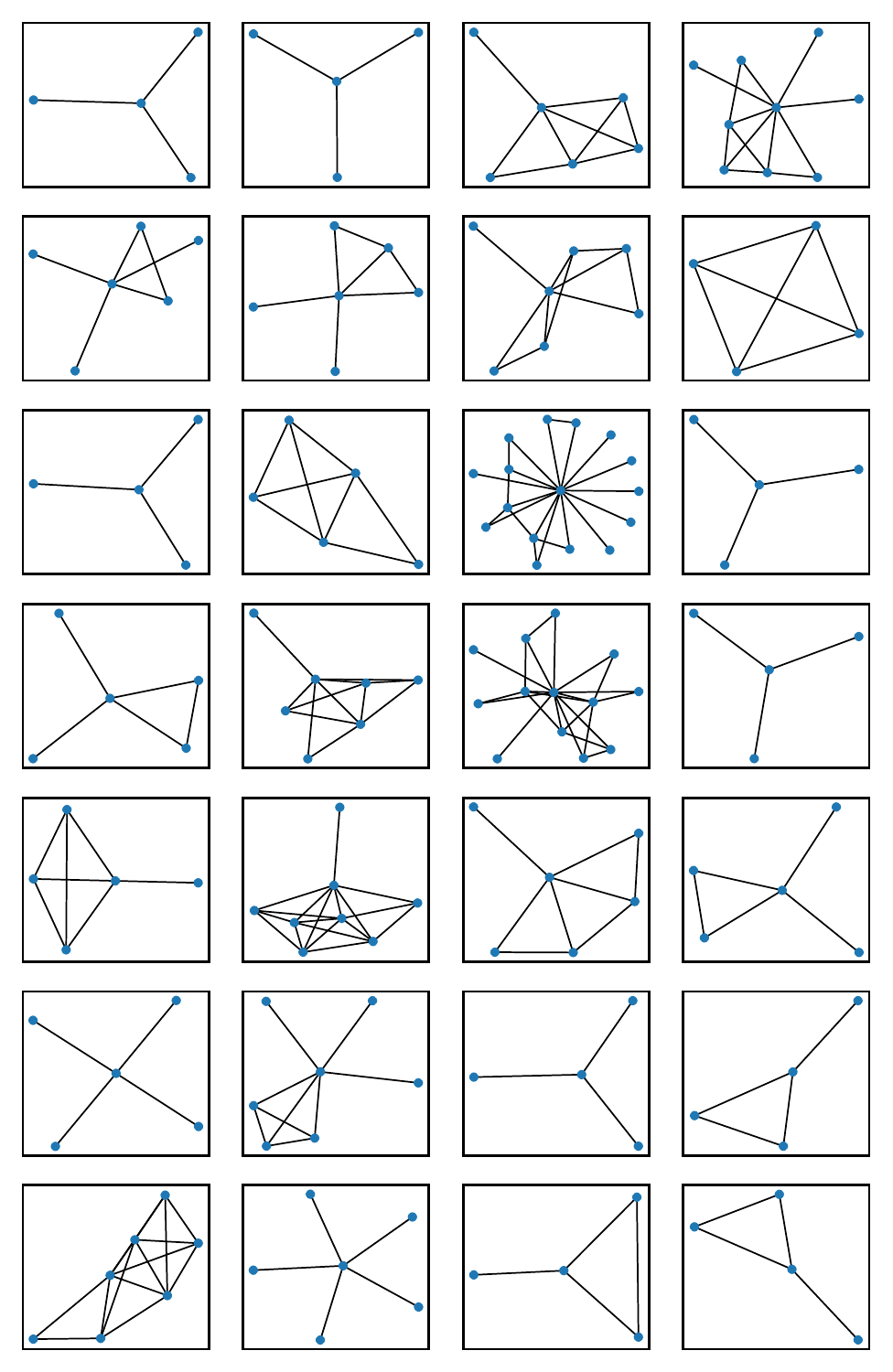}
    }  \quad
    \subfigure[GraphRNN samples]{
        \includegraphics[width=.27\linewidth]{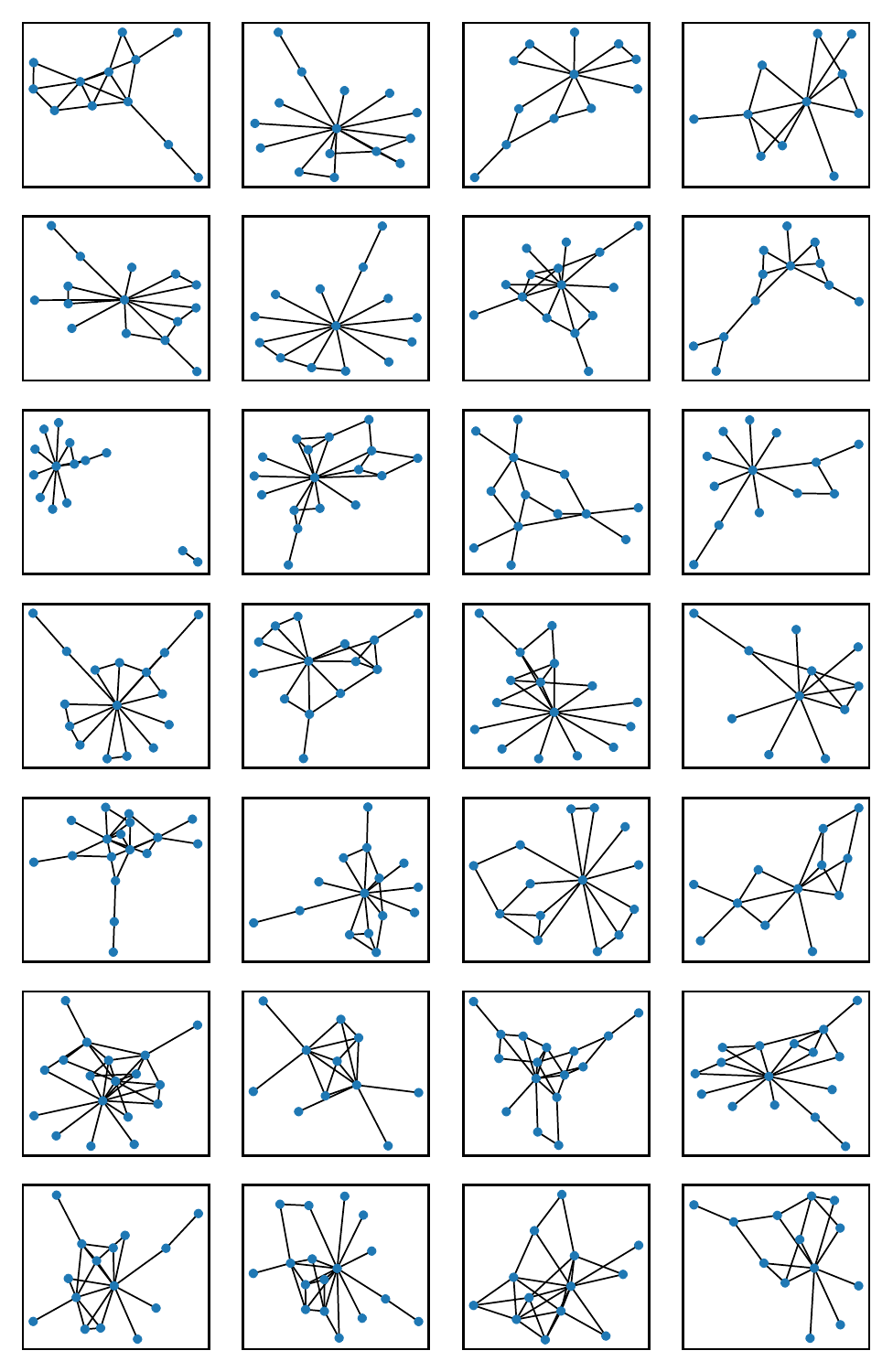}
    }
\caption{Extra samples from the training data, \MODEL, and GraphRNN, on Ego-small.}
\end{figure*}

\begin{figure*}[ht]
    \centering
	\subfigure[Training data]{
        \includegraphics[width=.27\linewidth]{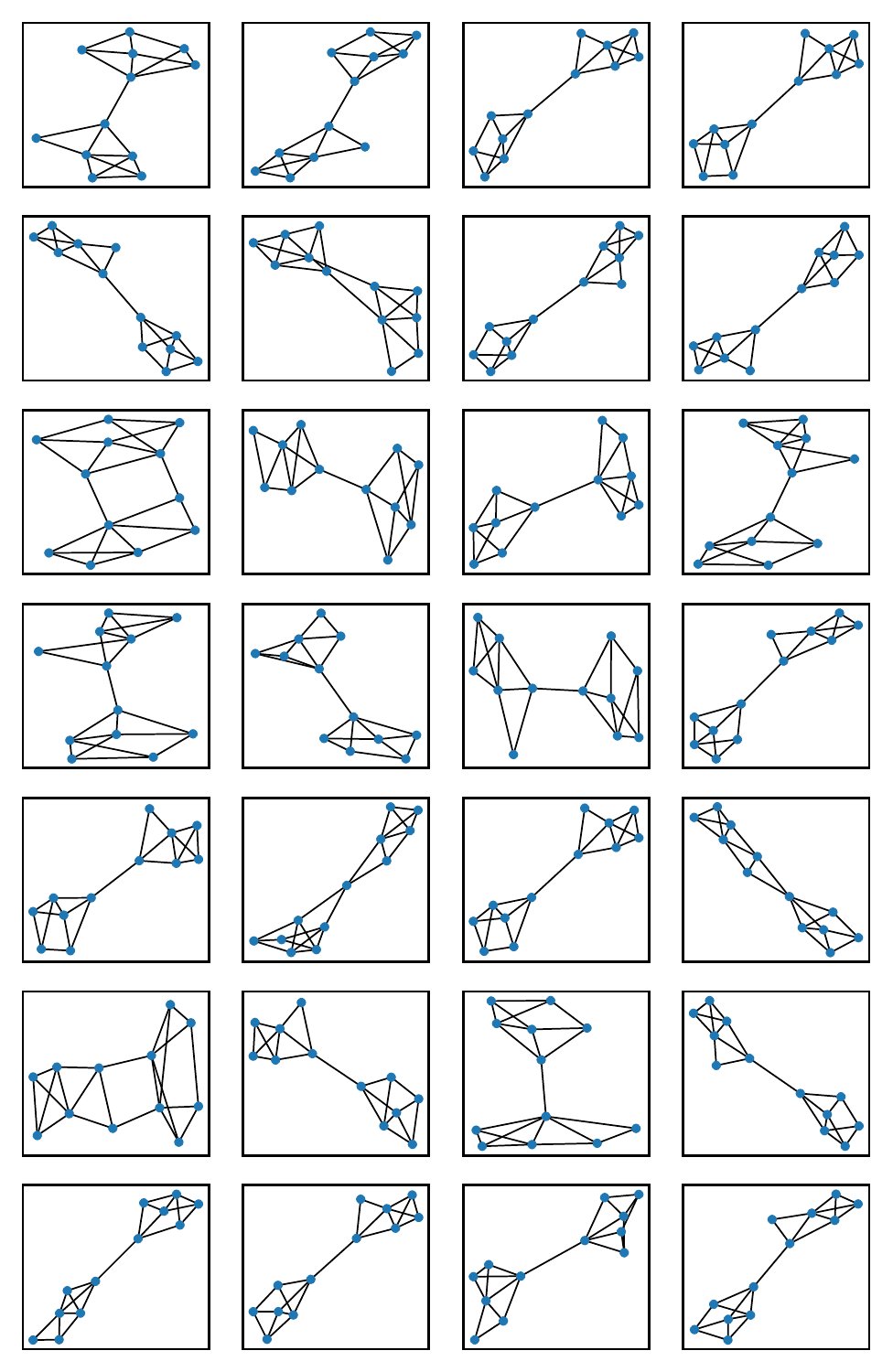}
    } \quad
    \subfigure[\MODEL \space samples]{
        \includegraphics[width=.27\linewidth]{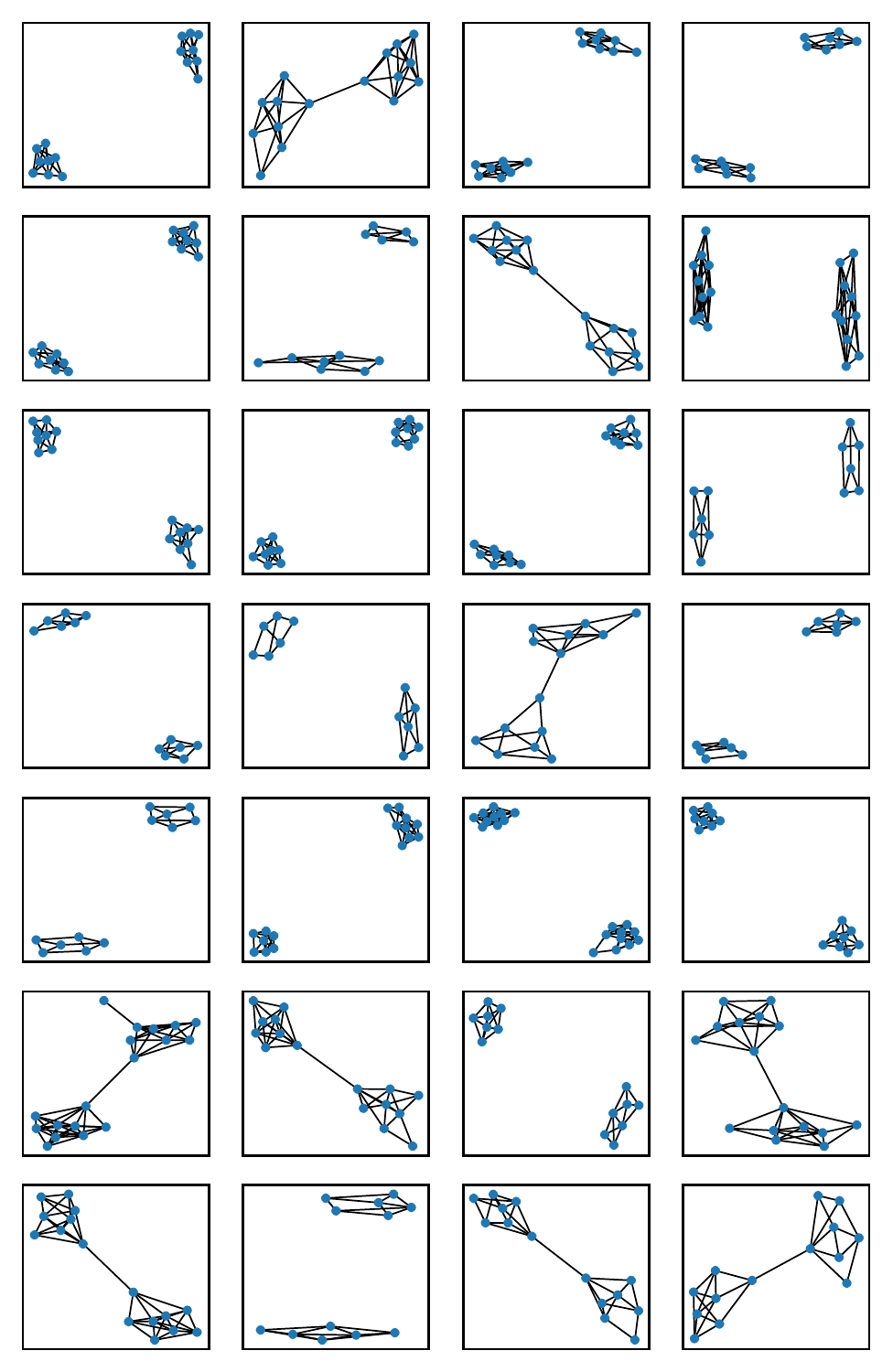}
    }  \quad
    \subfigure[GraphRNN samples]{
        \includegraphics[width=.27\linewidth]{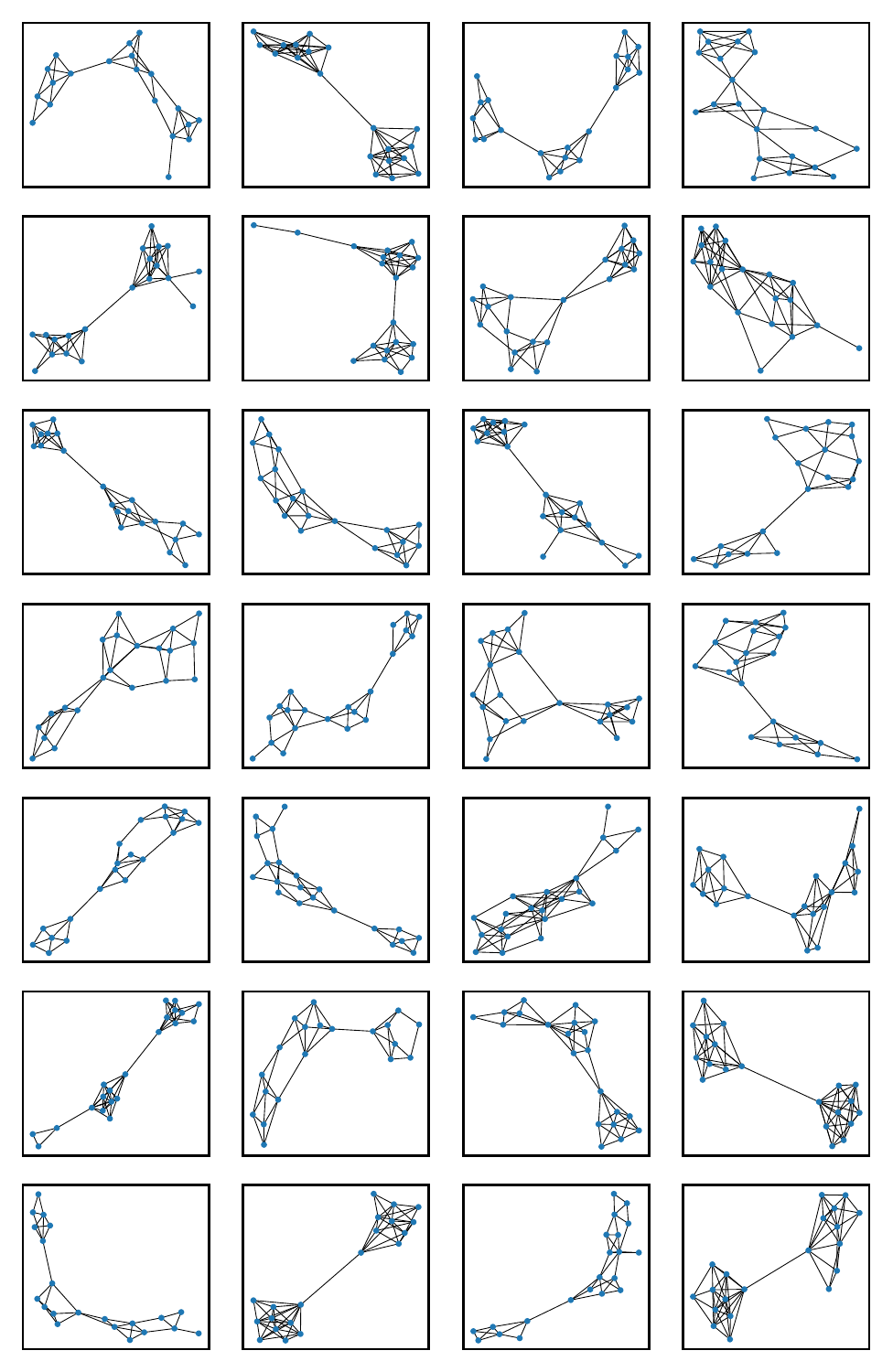}
    }
\caption{Extra samples from the training data, \MODEL, and GraphRNN, on Community-small.}

\end{figure*}

\begin{figure*}[ht]
    \centering
	\subfigure[Training data]{
        \includegraphics[width=.35\linewidth]{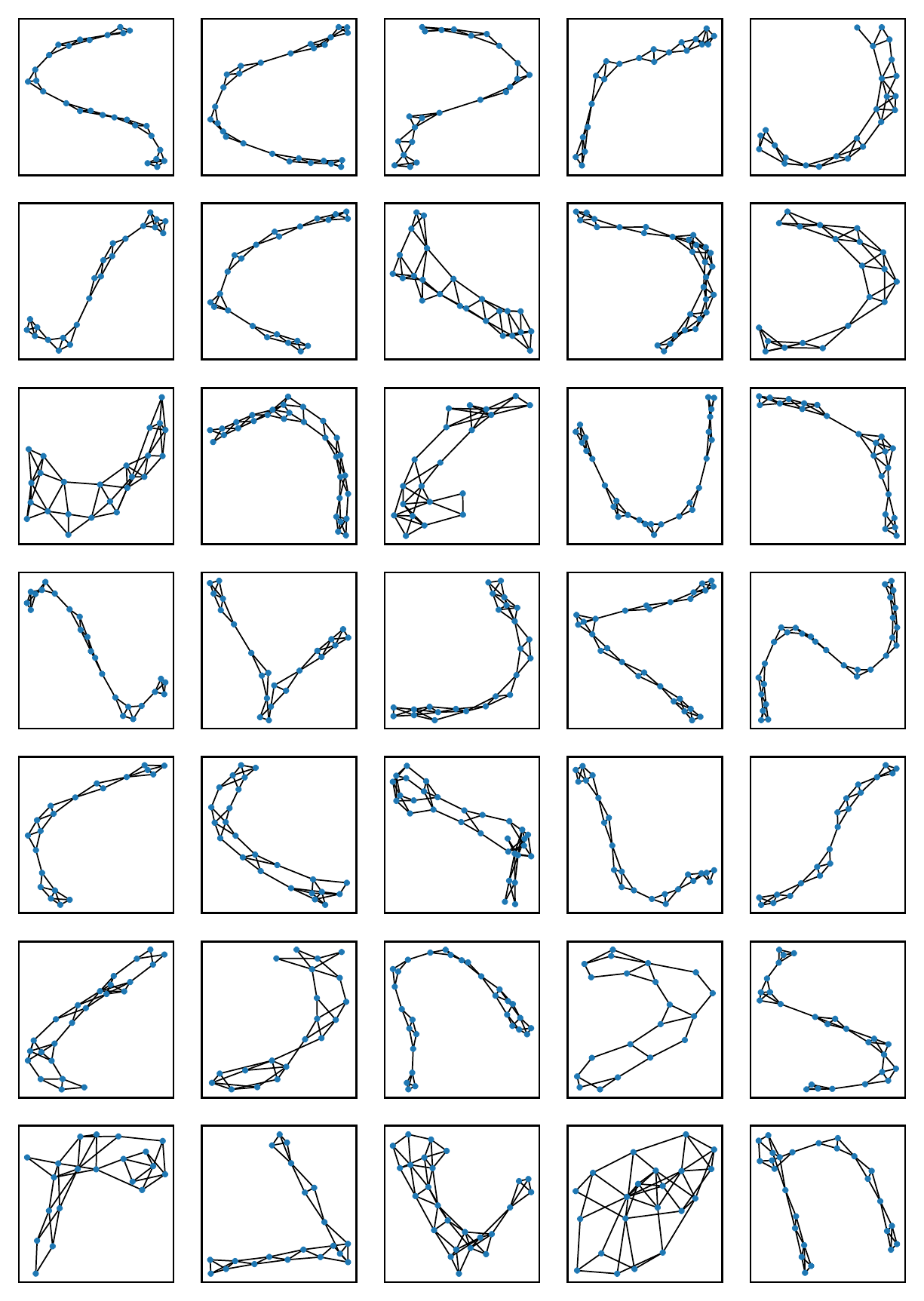}
    } \quad
    \subfigure[\MODEL \space samples]{
        \includegraphics[width=.35\linewidth]{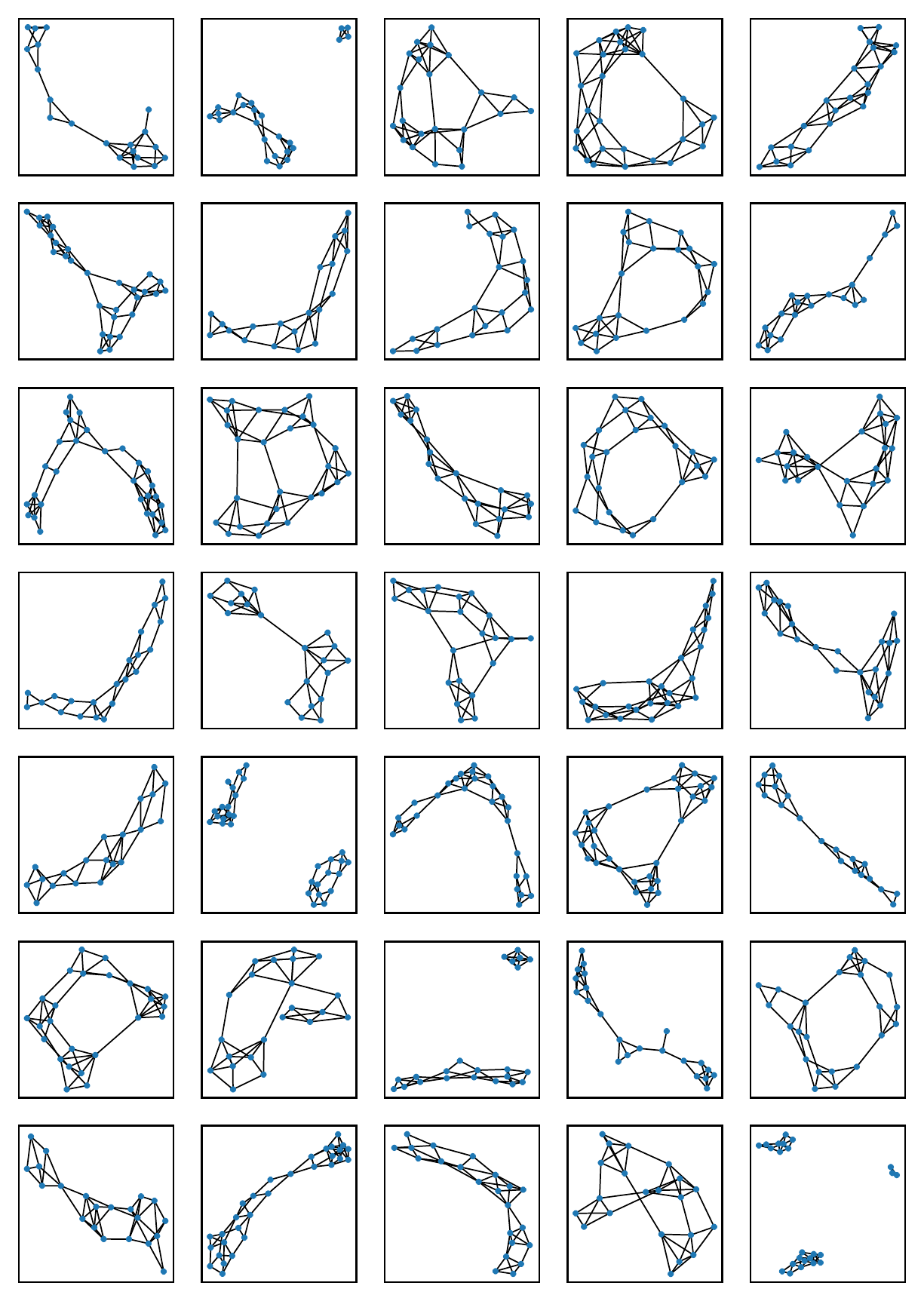}
    }
\caption{Extra samples from the training data and EDP-GNN, on the Protein dataset \citep{dobson2003distinguishing}, with the number of node $20 \le N \le 30$.}
\end{figure*}

\begin{figure*}[ht]
    \centering
	\subfigure[Training data]{
        \includegraphics[width=.35\linewidth]{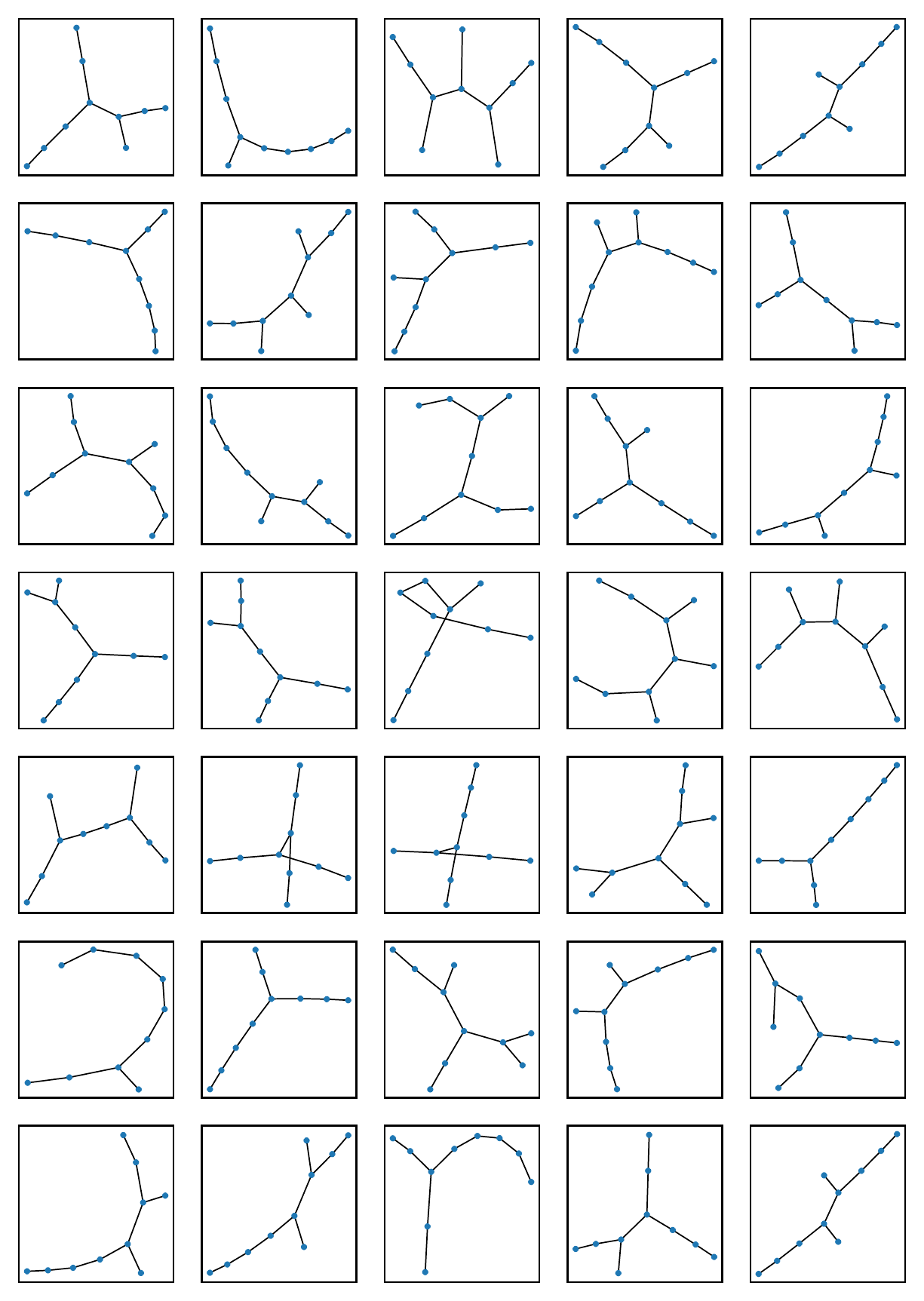}
    } \quad
    \subfigure[\MODEL \space samples]{
        \includegraphics[width=.35\linewidth]{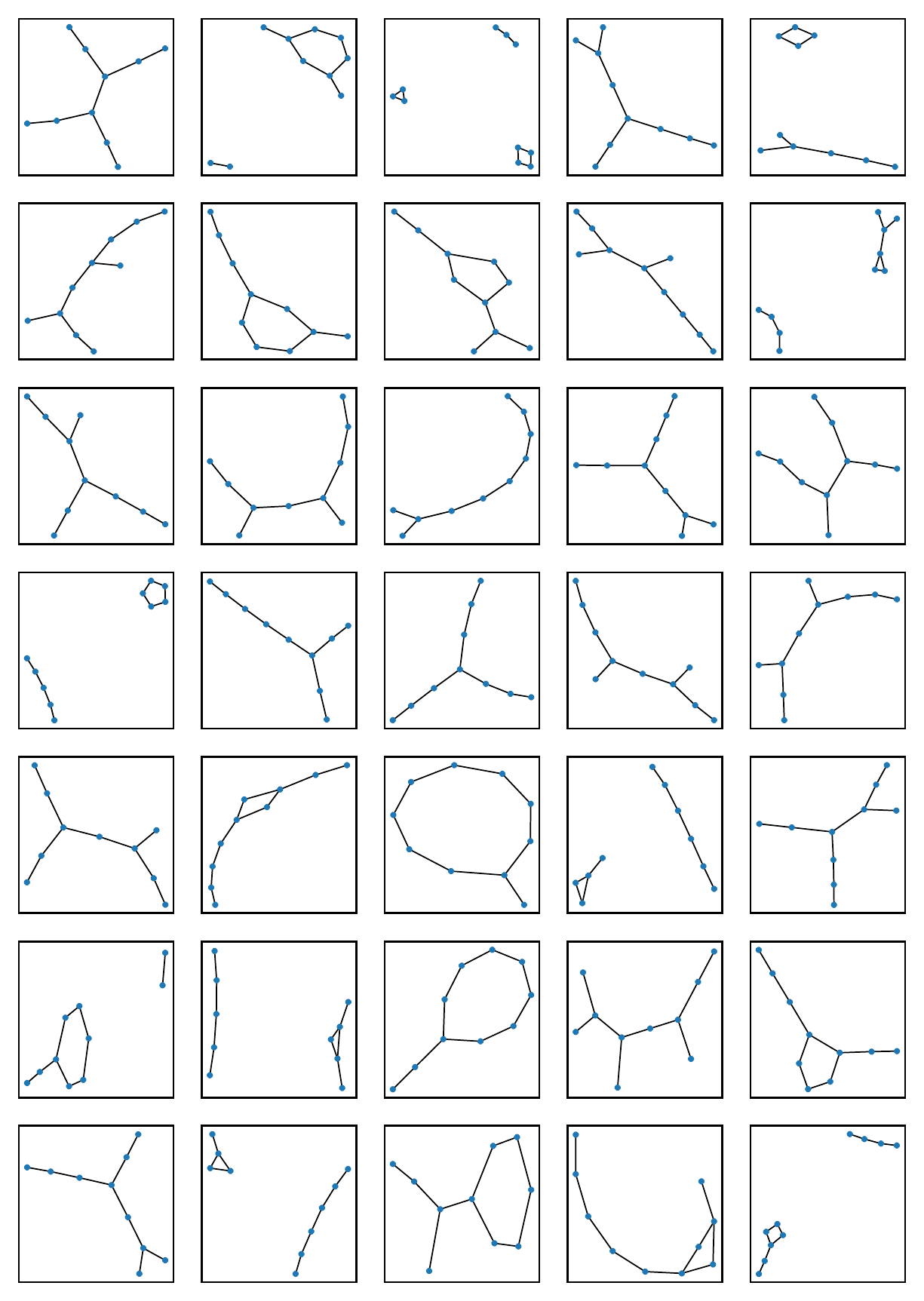}
    }
\caption{Extra samples from the training data and EDP-GNN, on the Lobster graph dataset \citep{golomb1996polyominoes}, with the number of node $N = 10$.}
\end{figure*}

\end{document}